\documentclass
[
    twoside,                 
    openright,               
    cleardoublepage = empty, 
    fontsize = 12 pt,        
    american,                
    captions = tableheading, 
    numbers = noenddot,      
    footheight = 35 pt,      
]
{scrbook}

\newif\ifprintVersion   
\newif\ifprofessionalPrint 
\newif\iffancyTheorems  
\newif\ifboldNumberSets 
\newif\ifbachelorThesis 

\printVersionfalse
\professionalPrintfalse
\fancyTheoremstrue
\boldNumberSetstrue
\bachelorThesistrue

\newcommand*{\printTitle}{}
\newcommand*{\printGermanTitle}{}
\newcommand*{\myTitle}[2]{\renewcommand*{\printTitle}{#1}\renewcommand*{\printGermanTitle}{#2}}
\newcommand*{\printTitleBold}{\textbf{\printTitle}}

\newcommand*{\printAuthor}{}
\newcommand*{\myName}[1]{\renewcommand*{\printAuthor}{#1}}

\newcommand*{\printProgram}{}
\newcommand*{\myProgram}[1]{\renewcommand*{\printProgram}{#1}}

\newcommand*{\printDateReceived}{}
\newcommand*{\dateOfHandingIn}[1]{\renewcommand*{\printDateReceived}{#1}}

\newcommand*{\printSubject}{}
\newcommand*{\mySubject}[1]{\renewcommand*{\printSubject}{#1}}

\newcommand*{\printKeywords}{}
\newcommand*{\myKeywords}[1]{\renewcommand*{\printKeywords}{#1}}

\newcommand*{\printNameOfSupervisor}{}
\newcommand*{\nameOfMySupervisor}[1]{\renewcommand*{\printNameOfSupervisor}{#1}}

\newcommand*{\printAdditionalExaminers}{}
\newcommand*{\additionalExaminers}[1]{\renewcommand*{\printAdditionalExaminers}{#1}}

\newlength{\extraborderlength}
\newcommand*{\extraBorder}[1]{\setlength{\extraborderlength}{#1}}

\newlength{\mybindingcorrection}
\newcommand*{\bindingCorrection}[1]{\setlength{\mybindingcorrection}{#1}} 

    \extraBorder{3 mm}

    \bindingCorrection{6 mm}

    \myTitle{Local Search on Vertex Coloring for Bipartite Graphs}{Lokale Suche zur Knotenfärbung auf Bipartiten Graphen}

    \myName{Johanna Gasse}

    \myProgram{IT-Systems Engineering}

    \dateOfHandingIn{1. September 2025}

    \nameOfMySupervisor{Prof.\,Dr. Tobias Friedrich}

    \additionalExaminers{Dr. Timo Kötzing\newline Aishwarya Radhakrishnan}

    \mySubject{A cool bachelor/master thesis.}

    \myKeywords{bachelor thesis | evolutionary algorithms | random local search | gray-box}

\usepackage[utf8]{inputenc} 
\usepackage[T1]{fontenc}    
\usepackage
[
    ngerman,         
    main = american, 
]
{babel}                     

\input glyphtounicode
\usepackage{calc} 

\newlength{\myparindent}
\newlength{\myparskip}
\setfootnoterule{0 cm}

\deffootnote[1.2 em]{1.2 em}{0 em}{\makebox[1.4 em][l]{\textbf{\thefootnotemark}}}

\makeatletter%
    \@removefromreset{footnote}{chapter}%
\makeatother

\usepackage[dvipsnames]{xcolor} 

\definecolor{stroke1}{HTML}{2574A9} 

\colorlet{captionlabel}{black}
\colorlet{footerpagenr}{black}
\colorlet{footerchapter}{stroke1}
\colorlet{footerchaptername}{black}
\colorlet{footersection}{stroke1}
\colorlet{footersectionname}{black}
\colorlet{chapternumber}{stroke1}

\newlength{\mypaperwidth}
\newlength{\mypaperheight}
\newlength{\mybodywidth}
\newlength{\mybodyheight}
\newlength{\myoutermargin}
\ifprintVersion
    \ifprofessionalPrint
    \else
    \fi
\else
\fi

\newlength{\mytopmargin}
\ifprintVersion
    \ifprofessionalPrint
    \fi
\fi

\newlength{\myinnermargin}
\ifprintVersion
    \ifprofessionalPrint
    \fi
\fi

\newlength{\mybottommargin}
\ifprintVersion
    \ifprofessionalPrint
    \fi
\fi

\newcommand{\goldenratio}{1.618}

\newlength{\myheadsep} 
\ifprintVersion
    \ifprofessionalPrint
    \fi
\fi

\newlength{\myfootskip} 
\ifprintVersion
    \ifprofessionalPrint
    \fi
\fi

\newlength{\mymargininnersep} 
\newlength{\mymarginoutersep} 
\ifprintVersion
    \ifprofessionalPrint
    \fi
\fi

\newlength{\mymarginwidth} 
\newlength{\mymarginwidthwithinnersep} 
\usepackage
[
    \ifprintVersion
        \ifprofessionalPrint
            paperwidth = \mypaperwidth + 2\extraborderlength + \mybindingcorrection,
            paperheight = \mypaperheight + 2\extraborderlength,
        \else
            paperwidth = \mypaperwidth,
            paperheight = \mypaperheight,
        \fi
    \else
        paperwidth = \mypaperwidth,
        paperheight = \mypaperheight,
    \fi
    textwidth = \mybodywidth,
    textheight = \mybodyheight,
    outer = \myoutermargin,
    top = \mytopmargin,
    headsep = \myheadsep,
    footskip = \myfootskip,
    marginparsep = \mymargininnersep,
    marginparwidth = \mymarginwidth,
]
{geometry} 

\usepackage
[
]
{scrlayer-scrpage} 

\clearpairofpagestyles

\KOMAoptions
{%
    headwidth = \textwidth + \mymarginwidthwithinnersep,%
    footwidth = \myoutermargin : \textwidth,%
}

\automark[chapter]{chapter}
\automark*[section]{}

\lehead%
{%
    \begin{minipage}[b]{\mymarginwidth}%
        \small\raggedleft\normalfont\textsf{\textbf{\color{footerchapter}\chaptername\ \thechapter}}
    \end{minipage}
}
\cehead{\hspace*{\mymarginwidthwithinnersep}\parbox{\textwidth}{\raggedright\leftmark}}

\rohead%
{%
    \Ifstr{\rightmark}{\leftmark}%
    {%
        \begin{minipage}[b]{\mymarginwidth}%
            \small\raggedright\normalfont\textsf{\textbf{\color{footersection}Chapter\ \thechapter}}%
        \end{minipage}%
    }%
    {%
        \begin{minipage}[b]{\mymarginwidth}%
            \small\raggedright\normalfont\textsf{\textbf{\color{footersection}Section\ \thesection}}%
        \end{minipage}%
    }%
}
\cohead{\hspace*{-\mymarginwidthwithinnersep}\parbox{\textwidth}{\raggedleft\rightmark}}

\lefoot*%
{%
    \vspace*{1 ex}%
    {\color{stroke1}\rule{\myoutermargin - \mymargininnersep}{0.5 mm}}\\
    \begin{minipage}[b]{\myoutermargin - \mymargininnersep}%
        \raggedleft\normalfont\color{footerpagenr}\textbf{\thepage}%
    \end{minipage}%
}
\rofoot*%
{%
    {\color{stroke1}\rule{\myoutermargin - \mymargininnersep}{0.5 mm}}\\
    \begin{minipage}[b]{\myoutermargin - \mymargininnersep}%
        \raggedright\normalfont\color{footerpagenr}\textbf{\thepage}%
    \end{minipage}%
}

\usepackage{caption}
\usepackage{tikz} 

\newlength{\mytmpa}
\newlength{\mytmpb}

\renewcommand*{\partlineswithprefixformat}[3]%
{%
    #2
    \thispagestyle{empty}
    \setlength{\mytmpa}{0.618\mypaperwidth}%
    \setlength{\mytmpb}{0.382\mypaperheight}%
    \ifprintVersion
        \ifprofessionalPrint
            \setlength{\mytmpa}{0.618\mypaperwidth + \mybindingcorrection + \extraborderlength}%
            \setlength{\mytmpb}{0.382\mypaperheight + \extraborderlength}%
        \fi
    \fi
    \begin{tikzpicture}[overlay, remember picture]%
        \node [inner sep = 0, outer sep = 0, anchor = north] at (current page.north west)%
        {%
            \begin{tikzpicture}[overlay, remember picture]%
            \draw[color = stroke1, line width = 0.7 mm] (\mytmpa, 0) -- (\mytmpa, -\mytmpb);%
            \end{tikzpicture}%
        };%
        \node (align) [align = right, below = \mytmpb - 2 ex, inner sep = 0, outer sep = 0, anchor = north west] at (current page.north west)%
        {%
            \hspace{\mytmpa}\hspace{0.5 em}\partname\ \thepart\\[1 ex]
            \color{stroke1}#3%
        };%
    \end{tikzpicture}%
}
\RedeclareSectionCommand%
[%
    font = \normalfont\Huge\sffamily,
    prefixfont = \normalfont\Huge\sffamily,
]
{part}

\usepackage{etoolbox}

\newbool{chapterHasANumber}
\newbool{chapterHasAStar}
\renewcommand*{\chapterlinesformat}[3]%
{%
    \Ifnumbered{#1}{\setbool{chapterHasANumber}{true}}{\setbool{chapterHasANumber}{false}}%
    \Ifstr{#2}{}{\setbool{chapterHasAStar}{true}}{\setbool{chapterHasAStar}{false}}%
    \ifboolexpr{bool{chapterHasANumber} and not bool{chapterHasAStar}}%
    {%
        \begin{tikzpicture}[overlay, remember picture]%
            \node [right = \myinnermargin, below = \mytopmargin, inner sep = 0, outer sep = 0, anchor = north west] (numbernode) at (current page.north west)%
            {%
                \hspace{\myinnermargin}%
                \sffamily\fontsize{60}{60}\selectfont%
                \color{chapternumber}%
                \thechapter%
            };%
            \node [inner sep = 0, outer sep = 0, anchor = north west] at (numbernode.south west)%
            {%
                \begin{tikzpicture}[overlay, remember picture]%
                    \draw[color = stroke1, line width = 0.7 mm] (\myinnermargin, -1 ex) -- (\paperwidth, -1 ex);%
                \end{tikzpicture}%
            };%
            \node (align) [text width = \textwidth - 2 cm, align = right, right = \myinnermargin + \mybodywidth, inner sep = 0, outer sep = 0, anchor = east] at (numbernode.west)%
            {%
                #3%
            };%
        \end{tikzpicture}%
    }%
    {%
        \begin{tikzpicture}[overlay, remember picture]%
            \node [right = \myinnermargin, below = \mytopmargin, inner sep = 0, outer sep = 0, anchor = north west] (numbernode) at (current page.north west)%
            {%
                \hspace{\myinnermargin}%
                \sffamily\fontsize{60}{60}\selectfont%
                \color{white}%
                \thechapter%
            };%
            \node [inner sep = 0, outer sep = 0, anchor = north west] at (numbernode.south west)%
            {%
                \begin{tikzpicture}[overlay, remember picture]%
                    \draw[color = stroke1, line width = 0.7 mm] (\myinnermargin, -1 ex) -- (\paperwidth, -1 ex);%
                \end{tikzpicture}%
            };%
            \node (align) [align = left, right = \myinnermargin, inner sep = 0, outer sep = 0, anchor = south west] at (numbernode.south west)%
            {%
                #3%
            };%
        \end{tikzpicture}%
    }%
}
\RedeclareSectionCommand%
[%
    font = \color{stroke1}\normalfont\huge\sffamily,
    afterskip = 20 pt,
]
{chapter}

\BeforeStartingTOC[toc]{\pagestyle{plain}}
\AfterStartingTOC{\thispagestyle{plain}}                        
\usepackage
[
    sortcites,              
    style = alphabetic,     
    defernumbers,           
    safeinputenc,           
    backref = true,         
    backrefstyle = three,   
    hyperref = true,        
    maxbibnames = 99,       
    maxcitenames = 2,       
]
{biblatex} 

\renewbibmacro{in:}%
{%
    \ifentrytype{article}{}{\printtext{\bibstring{in}\intitlepunct}}%
}

\renewbibmacro*{volume+number+eid}%
{%
    \printfield{volume}%
    \iffieldundef{number}{}{\addcolon}%
    \printfield{number}%
    \setunit*{\addcomma\space}%
    \printfield{eid}%
}

\DefineBibliographyStrings{english}%
{%
    backrefpage  = {\lowercase{s}ee page}, 
    backrefpages = {\lowercase{s}ee pages} 
}

\DeclareFieldFormat[article]{title}{\textbf{\color{stroke1}#1}}
\DeclareFieldFormat[inproceedings]{title}{\textbf{\color{stroke1}#1}}
\DeclareFieldFormat[thesis]{title}{\textbf{\color{stroke1}#1}}
\DeclareFieldFormat[book]{title}{\textbf{\color{stroke1}#1}}
\DeclareFieldFormat[unpublished]{title}{\textbf{\color{stroke1}#1}}
\DeclareFieldFormat[report]{title}{\textbf{\color{stroke1}#1}}
\DeclareFieldFormat[inbook]{chapter}{\textbf{\color{stroke1}#1}}
\DeclareFieldFormat[inbook]{title}{#1}
\DeclareFieldFormat{pages}{#1}

\newtoggle{authorend}
\togglefalse{authorend}

\DeclareBibliographyDriver{article}%
{%
  \usebibmacro{bibindex}%
  \usebibmacro{begentry}%
  \iftoggle{authorend}{}{\usebibmacro{author/translator+others}}%
  \setunit{\labelnamepunct}\newblock
  \usebibmacro{title}%
  \newunit
  \printlist{language}%
  \newunit\newblock
  \usebibmacro{byauthor}%
  \newunit\newblock
  \usebibmacro{bytranslator+others}%
  \newunit\newblock
  \printfield{version}%
  \newunit\newblock
  \usebibmacro{in:}%
  \usebibmacro{journal+issuetitle}%
  \newunit
  \usebibmacro{byeditor+others}%
  \newunit
  \usebibmacro{note+pages}%
  \newunit\newblock
  \iftoggle{bbx:isbn}
  {\printfield{issn}}
  {}%
  \newunit\newblock
  \usebibmacro{doi+eprint+url}%
  \newunit\newblock
  \usebibmacro{addendum+pubstate}%
  \setunit{\bibpagerefpunct}\newblock
  \usebibmacro{pageref}%
  \newunit\newblock
  \iftoggle{bbx:related}
  {\usebibmacro{related:init}%
    \usebibmacro{related}}
  {}%
  \usebibmacro{finentry}%
  \iftoggle{authorend}{\usebibmacro{author/translator+others}}{}%
}

\DeclareBibliographyDriver{inbook}%
{%
  \usebibmacro{bibindex}%
  \usebibmacro{begentry}%
  \iftoggle{authorend}{}{\usebibmacro{author/translator+others}}%
  \setunit{\labelnamepunct}\newblock
  \usebibmacro{chapter+pages}%
  \newunit
  \printlist{language}%
  \newunit\newblock
  \usebibmacro{byauthor}%
  \newunit\newblock
  \usebibmacro{in:}%
  \usebibmacro{bybookauthor}%
  \newunit\newblock
  \usebibmacro{maintitle+booktitle}%
  \newunit\newblock
  \usebibmacro{byeditor+others}%
  \newunit\newblock
  \printfield{edition}%
  \newunit
  \iffieldundef{maintitle}
  {\printfield{volume}%
    \printfield{part}}
  {}%
  \newunit
  \printfield{volumes}%
  \newunit\newblock
  \usebibmacro{series+number}%
  \newunit\newblock
  \printfield{note}%
  \newunit\newblock
  \usebibmacro{publisher+location+date}%
  \newunit\newblock
  \newunit\newblock
  \iftoggle{bbx:isbn}
  {\printfield{isbn}}
  {}%
  \newunit\newblock
  \usebibmacro{doi+eprint+url}%
  \newunit\newblock
  \usebibmacro{addendum+pubstate}%
  \setunit{\bibpagerefpunct}\newblock
  \usebibmacro{pageref}%
  \newunit\newblock
  \iftoggle{bbx:related}
  {\usebibmacro{related:init}%
    \usebibmacro{related}}
  {}%
  \usebibmacro{finentry}%
  \iftoggle{authorend}{\usebibmacro{author/translator+others}}{}%
}

\DeclareBibliographyDriver{inproceedings}%
{%
  \usebibmacro{bibindex}%
  \usebibmacro{begentry}%
  \iftoggle{authorend}{}{\usebibmacro{author/translator+others}}%
  \setunit{\labelnamepunct}\newblock
  \usebibmacro{title}%
  \newunit
  \printlist{language}%
  \newunit\newblock
  \usebibmacro{byauthor}%
  \newunit\newblock
  \usebibmacro{in:}%
  \usebibmacro{maintitle+booktitle}%
  \newunit\newblock
  \usebibmacro{event+venue+date}%
  \newunit\newblock
  \usebibmacro{byeditor+others}%
  \newunit\newblock
  \iffieldundef{maintitle}
  {\printfield{volume}%
    \printfield{part}}
  {}%
  \newunit
  \printfield{volumes}%
  \newunit\newblock
  \usebibmacro{series+number}%
  \newunit\newblock
  \printfield{note}%
  \newunit\newblock
  \printlist{organization}%
  \newunit
  \usebibmacro{publisher+location+date}%
  \newunit\newblock
  \usebibmacro{chapter+pages}%
  \newunit\newblock
  \iftoggle{bbx:isbn}
  {\printfield{isbn}}
  {}%
  \newunit\newblock
  \usebibmacro{doi+eprint+url}%
  \newunit\newblock
  \usebibmacro{addendum+pubstate}%
  \setunit{\bibpagerefpunct}\newblock
  \usebibmacro{pageref}%
  \newunit\newblock
  \iftoggle{bbx:related}
  {\usebibmacro{related:init}%
    \usebibmacro{related}}
  {}%
  \usebibmacro{finentry}%
  \iftoggle{authorend}{\usebibmacro{author/translator+others}}{}%
}

\DeclareBibliographyDriver{thesis}%
{%
  \usebibmacro{bibindex}%
  \usebibmacro{begentry}%
  \iftoggle{authorend}{}{\usebibmacro{author}}%
  \setunit{\labelnamepunct}\newblock
  \usebibmacro{title}%
  \newunit
  \printlist{language}%
  \newunit\newblock
  \usebibmacro{byauthor}%
  \newunit\newblock
  \printfield{note}%
  \newunit\newblock
  \printfield{type}%
  \newunit
  \usebibmacro{institution+location+date}%
  \newunit\newblock
  \usebibmacro{chapter+pages}%
  \newunit
  \printfield{pagetotal}%
  \newunit\newblock
  \iftoggle{bbx:isbn}
  {\printfield{isbn}}
  {}%
  \newunit\newblock
  \usebibmacro{doi+eprint+url}%
  \newunit\newblock
  \usebibmacro{addendum+pubstate}%
  \setunit{\bibpagerefpunct}\newblock
  \usebibmacro{pageref}%
  \newunit\newblock
  \iftoggle{bbx:related}
  {\usebibmacro{related:init}%
    \usebibmacro{related}}
  {}%
  \usebibmacro{finentry}%
  \iftoggle{authorend}{\usebibmacro{author}}{}%
}

\providebool{bbx:subentry}
\newbibmacro*{citenum}%
{
  \printtext[bibhyperref]{
    \printfield{prefixnumber}%
    \printfield{labelnumber}%
    \ifbool{bbx:subentry}
    {\printfield{entrysetcount}}
    {}}%
}

\DeclareCiteCommand{\conline}[\mkbibbrackets]
{\usebibmacro{prenote}}
{\usebibmacro{citeindex}%
  \usebibmacro{citenum}}
{\multicitedelim}
{\usebibmacro{postnote}}       
\usepackage
[
    babel = true, 
]
{microtype}           
\usepackage{csquotes} 

\usepackage{amsmath}
\usepackage{amssymb}
\usepackage{amsthm}
\usepackage{thmtools}
\usepackage{mathtools}
\usepackage{thm-restate}
\usepackage{dsfont}        
\usepackage{braceMnSymbol} 

\usepackage
[
    ttscale = 0.85, 
]
{libertine} 
\usepackage
[
    libertine,    
    slantedGreek, 
    vvarbb,       
    libaltvw,     
]
{newtxmath} 
\usepackage{url} 
\usepackage{bm}  

\usepackage{graphicx} 
\usepackage
[
    subrefformat = simple, 
    labelformat = simple,  
]
{subcaption}         
\usepackage{wrapfig} 

\columnsep = \mymargininnersep

\usepackage{array}     
\usepackage{booktabs}  
\usepackage{longtable} 
\usepackage{pdflscape} 

\usepackage{enumitem} 

\usepackage
[
    ruled,         
    vlined,        
    linesnumbered, 
]
{algorithm2e} 

\usepackage{xspace}   
\usepackage
[
    shortcuts, 
]
{extdash}             
\usepackage{setspace} 

\usepackage{xparse}    
\usepackage{footnote}  
\usepackage{afterpage} 
\usepackage
[
    textsize = scriptsize, 
]
{todonotes}            

\usepackage{tikz}
\usetikzlibrary{positioning}

\usepackage
[
    bookmarks = true,                 
    bookmarksopen = false,            
    bookmarksnumbered = true,         
    pdfstartpage = 1,                 
    pdftitle = {{\printTitle}},       
    pdfauthor = {{\printAuthor}},     
    pdfsubject = {{\printSubject}},   
    pdfkeywords = {{\printKeywords}}, 
    breaklinks = true,                
    \ifprintVersion
        hidelinks,                    
    \else
    colorlinks = true,            
    allcolors = stroke1,          
    \fi
]
{hyperref} 

\usepackage
[
    noabbrev,   
    nameinlink, 
]
{cleveref} 
\newcommand*{\colloquialDegreeName}{Master}
\newcommand*{\colloquialDegreeNameLowercase}{master}

\newcommand*{\degreeAbbreviation}{M.}

\ifbachelorThesis
    \renewcommand*{\colloquialDegreeName}{Bachelor}
    \renewcommand*{\colloquialDegreeNameLowercase}{bachelor}
    \renewcommand*{\degreeAbbreviation}{B.}
\fi

\makeatletter
    \def\IfEmptyTF#1%
    {%
        \if\relax\detokenize{#1}\relax%
            \expandafter\@firstoftwo%
        \else%
            \expandafter\@secondoftwo%
        \fi%
    }
\makeatother

\NewDocumentCommand{\mathOrText}{m}
{%
    \ensuremath{#1}\xspace%
}

\let\originalleft\left
\let\originalright\right
\renewcommand{\left}{\mathopen{}\mathclose\bgroup\originalleft}
\renewcommand{\right}{\aftergroup\egroup\originalright}

\makeatletter
    \DeclareRobustCommand{\bfseries}%
    {%
        \not@math@alphabet\bfseries\mathbf%
        \fontseries\bfdefault\selectfont%
        \boldmath%
    }
\makeatother

\xspaceaddexceptions{]\}}

\allowdisplaybreaks

\crefname{ineq}{inequality}{inequalities}
\creflabelformat{ineq}{#2{\upshape(#1)}#3} 

\crefname{term}{term}{terms}
\creflabelformat{term}{#2{\upshape(#1)}#3}

\let\oldfootnote\footnote

\newlength{\spaceBeforeFootnote} 
\newlength{\spaceAfterFootnote}  

\RenewDocumentCommand{\footnote}{o o o m}%
{%
    \IfNoValueTF{#1}%
    {%
        \oldfootnote{#4}%
    }%
    {%
        \setlength{\spaceBeforeFootnote}{\IfEmptyTF{#1}{0}{#1} em}%
        \IfNoValueTF{#2}%
        {%
            \hspace*{\spaceBeforeFootnote}\oldfootnote{#4}%
        }%
        {%
            \setlength{\spaceAfterFootnote}{\IfEmptyTF{#2}{0}{#2} em}%
            \hspace*{\spaceBeforeFootnote}\IfNoValueTF{#3}{\oldfootnote{#4}}{\oldfootnote[#3]{#4}}\hspace*{\spaceAfterFootnote}%
        }%
    }%
}

\makesavenoteenv{figure}
\makesavenoteenv{table}
\makesavenoteenv{tabular}

\iffancyTheorems
    \declaretheoremstyle
    [
        spaceabove = \topsep,
        spacebelow = \topsep,
        headfont = \bfseries,
        headformat = \textcolor{stroke1}{$\blacktriangleright$} \NAME~\NUMBER \NOTE,
        notefont = \bfseries,
        notebraces = {(}{)},
        bodyfont = \normalfont,
        postheadspace = 0.5 em,
        qed = \textcolor{stroke1}{\bfseries$\blacktriangleleft$},
    ]
    {myTheoremStyle}
    
    \declaretheorem
    [
        style = myTheoremStyle,
        name = Conjecture,
        numberwithin = chapter,
    ]
    {conjecture}
    \declaretheorem
    [
        style = myTheoremStyle,
        name = Proposition,
        sharenumber = conjecture,
    ]
    {proposition}
    \declaretheorem
    [
        style = myTheoremStyle,
        name = Claim,
        sharenumber = conjecture,
    ]
    {claim}
    \declaretheorem
    [
        style = myTheoremStyle,
        name = Lemma,
        sharenumber = conjecture,
    ]
    {lemma}
    \declaretheorem
    [
        style = myTheoremStyle,
        name = Corollary,
        sharenumber = conjecture,
    ]
    {corollary}
    \declaretheorem
    [
        style = myTheoremStyle,
        name = Theorem,
        sharenumber = conjecture,
    ]
    {theorem}
    \declaretheorem
    [
        style = myTheoremStyle,
        name = Definition,
        sharenumber = conjecture,
    ]
    {definition}
    \declaretheorem
    [
        style = myTheoremStyle,
        name = Example,
        sharenumber = conjecture,
    ]
    {example}
    \declaretheorem
    [
        style = myTheoremStyle,
        name = Remark,
        sharenumber = conjecture,
    ]
    {remark}
\else
    \theoremstyle{plain}
    
    \newtheorem{conjecture}{Conjecture}[chapter]

    \newtheorem{lemma}[conjecture]{Lemma}
    
    \newtheorem{theorem}[conjecture]{Theorem}
    \newtheorem{definition}[conjecture]{Definition}

\fi

\NewDocumentCommand{\functionTemplate}{m m m m o}%
{%
    \IfNoValueTF{#5}%
    {%
        \mathOrText{#1\left#2{#4}\right#3}%
    }%
    {%
        \mathOrText{#1#5#2{#4}#5#3}%
    }%
}

\newcommand*{\leftBracketType}{(}
\newcommand*{\rightBracketType}{)}

\NewDocumentCommand{\createFunction}{m m o o}%
{%
    \renewcommand*{\leftBracketType}{\IfNoValueTF{#3}{(}{#3}}%
    \renewcommand*{\rightBracketType}{\IfNoValueTF{#4}{)}{#4}}%
    \NewDocumentCommand{#1}{o o}%
    {%
        \IfNoValueTF{##1}%
        {%
            \mathOrText{#2}%
        }%
        {%
            \functionTemplate{#2}{\leftBracketType}{\rightBracketType}{##1}[##2]%
        }%
    }%
}

\DeclareDocumentCommand{\probabilisticFunctionTemplate}{m m O{} o}
{%
    \functionTemplate{#1}%
    {\lbrack}%
    {\rbrack}%
    {#2\IfEmptyTF{#3}{}{\ \IfNoValueTF{#4}{\left}{#4}\vert\ \vphantom{#2}#3\IfNoValueTF{#4}{\right.}{}}}%
    [#4]%
}

\ifboldNumberSets

    \newcommand*{\indicatorFunctionSymbol}{\mathbf{1}}
\else

    \newcommand*{\indicatorFunctionSymbol}{\mathds{1}}
\fi

\RenewDocumentCommand{\Pr}{m O{} o}%
{%
    \probabilisticFunctionTemplate{\mathrm{Pr}}{#1}[#2][#3]%
}

\NewDocumentCommand{\E}{m O{} o}%
{%
    \probabilisticFunctionTemplate{\mathrm{E}}{#1}[#2][#3]%
}

\NewDocumentCommand{\Var}{m O{} o}%
{%
    \probabilisticFunctionTemplate{\mathrm{Var}}{#1}[#2][#3]%
}

\DeclareDocumentCommand{\bigO}{m o}%
{%
    \functionTemplate{\mathrm{O}}{(}{)}{#1}[#2]%
}

\DeclareDocumentCommand{\smallO}{m o}%
{%
    \functionTemplate{\mathrm{o}}{(}{)}{#1}[#2]%
}

\DeclareDocumentCommand{\bigTheta}{m o}%
{%
    \functionTemplate{\upTheta}{(}{)}{#1}[#2]%
}

\DeclareDocumentCommand{\bigOmega}{m o}%
{%
    \functionTemplate{\upOmega}{(}{)}{#1}[#2]%
}

\DeclareDocumentCommand{\smallOmega}{m o}%
{%
    \functionTemplate{\upomega}{(}{)}{#1}[#2]%
}

\DeclareDocumentCommand{\eulerE}{o}%
{%
    \mathOrText{\mathrm{e}\IfNoValueTF{#1}{}{^{#1}}}%
}

\DeclareDocumentCommand{\poly}{m o}%
{%
    \functionTemplate{\mathrm{poly}}{(}{)}{#1}[#2]%
}

\createFunction{\id}{\mathrm{id}}

\NewDocumentCommand{\ind}{m o o}%
{%
    \IfNoValueTF{#2}%
    {%
        \mathOrText{\indicatorFunctionSymbol_{#1}}%
    }%
    {%
        \functionTemplate{\indicatorFunctionSymbol_{#1}}{(}{)}{#2}[#3]%
    }%
}

\DeclareDocumentCommand{\dom}{m o}%
{%
    \functionTemplate{\mathrm{dom}}{(}{)}{#1}[#2]%
}

\DeclareDocumentCommand{\rng}{m o}%
{%
    \functionTemplate{\mathrm{rng}}{(}{)}{#1}[#2]%
}

\DeclareDocumentCommand{\d}{o}%
{%
    \mathrm{d}\IfNoValueTF{#1}{}{^{#1}}%
}

\DeclareDocumentCommand{\set}{m m o}%
{
    \mathOrText{\IfNoValueTF{#3}{\left}{#3}\{#1\ \IfNoValueTF{#3}{\left}{#3}\vert\
    \vphantom{#1}#2\IfNoValueTF{#3}{\right.}{}\IfNoValueTF{#3}{\right}{#3}\}}
}      

\DeclareDocumentCommand{\set}{m g o}{
    \ensuremath{
        \IfNoValueTF{#3}{\left}{#3}\{#1
            \IfNoValueTF{#2}{}{
                \ \IfNoValueTF{#3}{\left}{#3}\vert\ \vphantom{#1}#2\IfNoValueTF{#3}{\right.}{}
            } \IfNoValueTF{#3}{\right}{#3}\}
    }\xspace
}

\definecolor{lightgreen}{rgb}{0.75,0.92,0.61}

\newcommand{\labelname}[1]{
  \def\@currentlabelname{#1}}

\definecolor{mygreen}{RGB}{1, 150, 122}

\definecolor{flora}{RGB}{136, 57, 239}

\providecommand{\ignore}[1]{} 

\newcommand{\ooea}{(1+1) EA\xspace}

\newcommand{\col}{\mathrm{col}\xspace}

\newcommand{\Prob}[1]{\mathbb{P}\left[ #1 \right]}
\newcommand{\Ex}[1]{\mathbb{E}\left[ #1 \right]}

\DeclareMathOperator{\Geo}{Geo}

	{
	\begin{center}
	\begin{algorithm2e}
	}%
	{
	\end{algorithm2e}
	\end{center}
	}
	
\allowdisplaybreaks

\begin{document}

    \frontmatter

\ifprintVersion
    \ifprofessionalPrint
        \newgeometry
        {
            textwidth = 134 mm,
            textheight = 220 mm,
            top = 38 mm + \extraborderlength,
            inner = 38 mm + \mybindingcorrection + \extraborderlength,
        }
    \else
        \newgeometry
        {
            textwidth = 134 mm,
            textheight = 220 mm,
            top = 38 mm,
            inner = 38 mm + \mybindingcorrection,
        }
    \fi
\else
    \newgeometry
    {
        textwidth = 134 mm,
        textheight = 220 mm,
        top = 38 mm,
        inner = 38 mm,
    }
\fi

\begin{titlepage}
    \sffamily
    \begin{center}
        \includegraphics[height = 3.2 cm]{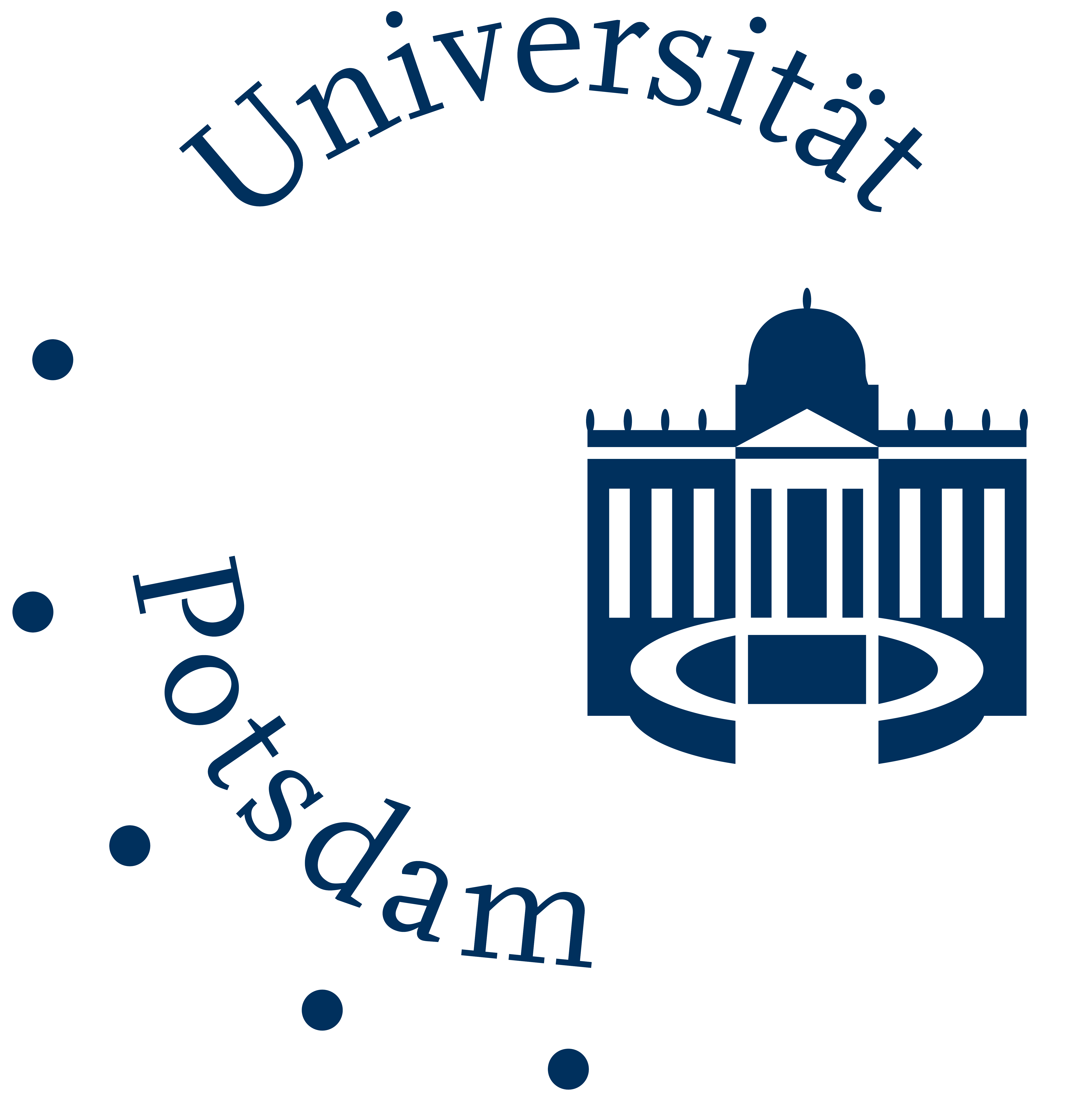} \hfill \includegraphics[height = 3 cm]{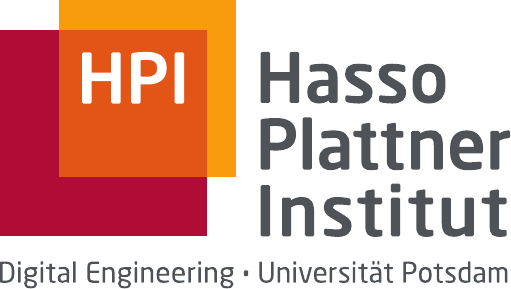}\\
        \vfil
        {\LARGE
            \rule[1 ex]{\textwidth}{1.5 pt}
            \onehalfspacing\printTitleBold\\[1 ex]
            {\vspace*{-1 ex}\Large \printGermanTitle}\\
            \rule[-1 ex]{\textwidth}{1.5 pt}
        }
        \vfil
        {\Large\textbf{\printAuthor}}
        \vfil
        {\large Universitäts\colloquialDegreeNameLowercase arbeit\\[0.25 ex]
        zur Erlangung des akademischen Grades}\\[0.25 ex]
        \bigskip
        {\Large \colloquialDegreeName{} of Science}\\[0.5 ex]
        {\large\emph{(\degreeAbbreviation\,Sc.)}}\\
        \bigskip
        {\large im Studiengang\\[0.25 ex]
        \printProgram}
        \vfil
        {\large eingereicht am \printDateReceived{} am\\[0.25 ex]
        Fachgebiet Algorithm Engineering der\\[0.25 ex]
        Digital-Engineering-Fakultät\\[0.25 ex]
        der Universität Potsdam}
    \end{center}
    
    \vfil
    \begin{table}[h]
        \centering
        \large
        \sffamily 
        {\def\arraystretch{1.2}
            \begin{tabular}{>{\bfseries}p{3.0 cm}p{6.1 cm}}
                Gutachter               & \printNameOfSupervisor\\
                Betreuer               & \printAdditionalExaminers
            \end{tabular}
        }
    \end{table}
\end{titlepage}

\restoregeometry

    \pagestyle{plain}

    \addchap{Abstract}

Local search is a well-known heuristic method used in optimization. In this thesis, we explore its capabilities on the vertex coloring problem, an $NP$-hard problem with relevance in both theoretical analysis and practical application. To recognize limitations in the applicability of local search of the vertex coloring problem, we analyze local search landscapes on differently-structured bipartite graphs. We identify structures that ensure only global optima can exist as well as ones that enable the existence of non-global local optima, showing that on general bipartite graphs, it is possible for local search to return arbitrarily bad results. Further, we analyze the capabilities of local search on graphs where a local optimum can be found. To do so, we introduce a gray-box local search mutation operator that removes less frequent colors with higher probability and prove that it finds an optimal coloring on complete bipartite graphs in an expected run time of $\Theta(n \log n)$. This is a drastic improvement to the exponential tun time of the black-box Random Local Search, showing that gray-box mutation operators can improve the run time of local search.

    \selectlanguage{ngerman}
    \addchap{Zusammenfassung}
    
Lokale Suche ist eine bekannte heuristische Methode, die bei Optimierungsproblemen verwendet wird. In dieser Arbeit untersuchen wir ihre Anwendungsmöglichkeiten bezogen auf das Knotenfärbungsproblem, ein $NP$-schweres Problem mit Relevanz in der theoretischen Analyse sowie praktischen Anwendung. Um die Grenzen der Anwendbarkeit von lokaler Suche auf das Knotenfärbungsproblem zu erkennen, analysieren wir Lokale-Suche-Landschaften auf unterschiedlich aufgebauten bipartiten Graphen. Wir identifizieren sowohl Strukturen in Graphen, die dazu führen, dass es nur globale Optima gibt, als auch solche, die nicht-globale lokale Optima hervorrufen. Damit zeigen wir, dass lokale Suche auf allgemeinen bipartiten Graphen beliebig schlechte Ergebnisse liefern kann. Des weiteren analysieren wir die Fähigkeiten lokaler Suche auf Graphen, bei denen sicher ein globales Optimum gefunden werden kann. Dazu führen wir einen Gray-Box Mutationsoperator ein, der seltenere Farben mit höherer Wahrscheinlichkeit entfernt, und zeigen, dass er in Erwartung in $\Theta(n \log n)$ Iterationen eine optimale Färbung auf vollständigen bipartiten Graphen findet. Dies ist eine drastische Verbesserung zur Laufzeit der Black-Box Random Local Search, womit wir zeigen, dass Gray-Box Mutationsoperatoren die Laufzeit lokaler Suche verbessern können.
    \selectlanguage{american}

    \addchap{Acknowledgments}
    
I am extremely grateful to my advisors, Timo and Aishwarya, for their wonderful help with writing this thesis. Thank you for the frequent meetings and discussions, for always being available when I needed advice and for the many insightful suggestions on how to improve my proofs.

I would also like to thank Antonia and Hendrik (as well as Timo once again for his wonderful support) for the many hours we spent working on gray-box optimization together. I really enjoyed diving into the subject and working on a paper with you!

Finally, thank you to all the members of my bachelor's project -- Alex, Antonia, Felix, Flora, Hendrik, Moritz, Paul and Philip -- for making this project and the work on the thesis as fun as it was.

    \setuptoc{toc}{totoc}
    \tableofcontents

    \pagestyle{headings}
    \mainmatter

    \chapter{Introduction}
    
\textit{Heuristics} are techniques for finding good solutions to optimization problems at a lower computational cost than alternative strategies. They are, as  \citeauthor{vos_metaheuristics_2008} \cite{vos_metaheuristics_2008} puts it, a \textquoteleft rule of thumb\textquoteright \space for how to search for optimal solutions. As such, heuristic approaches are not guaranteed to find an optimal solution, but can reduce the computational workload.

This makes heuristic techniques well-suited for $NP$-complete problems, which are very difficult to solve with exact algorithmic approaches. One such problem, which we intend to focus on in this work, is the \textit{vertex coloring} problem: Given a graph, the goal is to assign a color to each of its vertices such that no two adjacent vertices share the same color \cite{balakrishnan_graph_2012}. Finding the minimum number of colors for which such a coloring exists is known to be $NP$-complete \cite{korte_approximation_2018}. As vertex coloring finds frequent application in scheduling \cite{marx_graph_2004} and register allocation \cite{chaitin_register_1982}, where it is used to assign jobs to available resources, developing algorithms that are able to efficiently find an optimal coloring has a lot of practical relevance.

In this work, we analyze the performance of \textit{local search} on the vertex coloring problem on bipartite graphs. Local search is a basic heuristic strategy that works by iteratively altering solutions locally, meaning with only small changes between solutions, until an optimum is reached. However, due to the limited search space, it is known to get stuck in non-global local optima. This means that a local search algorithm might return a non-optimal solution if all neighboring solutions are worse -- even if another, better solution exists elsewhere \cite{rudolph_escaping_2022}.

According to \citeauthor{watson_introduction_2010} \cite{watson_introduction_2010}, local search heuristics are used widely, but it is not very well understood why or under what circumstances they work well. Most often, local search algorithms are empirically tested for their performance, but not mathematically analyzed. This \textquoteleft competitive testing\textquoteright \space is criticized by \citeauthor{hooker_testing_1995} \cite{hooker_testing_1995} for benefiting development more than research. 

Therefore, we first seek to gain some general understanding regarding the capabilities of local search on the vertex coloring problem. In \Cref{ch:landscape}, we analyze the local search landscape on different configurations of bipartite graphs for non-global local optima. We identify two features that ensure no non-global local optima can exist as well as two features that lead to a local optimum at $3$ and $k$ colors, respectively. The latter shows that local search can yield arbitrarily bad results on certain bipartite graphs. 

In \Cref{ch_kmn}, we shift our focus to complete bipartite graphs, a type of bipartite graphs shown in \Cref{ch:landscape} not to contain local optima. We develop a local search operator that removes less frequent colors with higher probability and prove that it finds a $2$-coloring on complete bipartite graphs in $\Theta(n \log n)$ in expectation. In doing so, we show that an operator tailored specifically to the problem -- a \emph{gray-box} operator -- can drastically outperform standard \emph{black-box} approaches such as Random Local Search, which take exponential time in expectation to solve the same problem.

\section{Vertex Coloring}

We now give a brief introduction to the vertex coloring problem. Due to its practical relevance in scheduling \cite{marx_graph_2004}, vertex coloring is a well-studied problem and we aim to provide an overview of current capabilities and limitations in solving certain instances of the problem.

We call a graph $k$-colorable if a proper $k$-coloring, meaning a coloring using $k$ colors and with no adjacent vertices colored in the same color, exists. For general graphs, it is $NP$-hard to decide whether such a graph is $k$-colorable or to find a $k$-coloring \cite{korte_approximation_2018}, which means that, assuming $P \ne NP$, the problem cannot be solved in polynomial time. 

However, one approach is to accept less-than-optimal colorings. \citeauthor{korte_approximation_2018} \cite{korte_approximation_2018} show that for any undirected graph with maximum degree $k$, a $(k+1)$-coloring exists and can be found in linear time.

Further, if certain properties of graphs are known, proper colorings can be found in polynomial time. One such class of graphs are planar graphs, which, intuitively, are graphs that can be drawn on paper without any edges intersecting.

It is proven by \citeauthor{diestel_colouring_2025} \cite{diestel_colouring_2025} that planar graphs can be $5$-colored. This was later expanded to a proof that all planar graphs are $4$-colorable, although the complexity of the proof can be inferred from the fact that the algorithmic version of the proof is 741 pages long. However, \citeauthor{korte_approximation_2018} \cite{korte_approximation_2018} gives a theorem by Stockmeyer [1973] that deciding whether a general planar graph is $3$-colorable is $NP$-hard. Still, structure in graphs can be exploited to solve this problem for certain instances. For example, Grötzsch [1959], as given by \citeauthor{diestel_colouring_2025} \cite{diestel_colouring_2025}, proves that every planar graph not containing a triangle is $3$-colorable.

Further, while finding a $3$-coloring is hard even on planar graphs, deciding whether a general graph can be colored with less than $3$ colors and, if so, finding such a coloring can be done in linear time as per \citeauthor{korte_approximation_2018} \cite{korte_approximation_2018}. This specifically means that deciding if a graph is $2$-colorable -- also called a \textit{bipartite graph} -- and finding a $2$-coloring can be accomplished in linear time.

\section{Local Search}

Local search in its most general form, as described by \citeauthor{crama_local_1995} \cite{crama_local_1995}, works by gradually transforming a given solution while improving the value of a \emph{fitness function} which measures how \textquoteleft good\textquoteright \space a solution is. Take the vertex coloring problem. There, we start with a proper coloring -- either randomly generated or set manually -- and gradually modify it while maintaining propriety. 

When performing modifications, also called \emph{mutations}, we only consider values in the \emph{neighborhood} of the current coloring. For example, a neighborhood might contain only colorings that can be created by changing the color of a single vertex. This focus on solutions that are close to the current candidate gives the approach the name local search and differentiates it from \emph{global search}, where all feasible solutions can be reached in a single mutation. 

When generating a new solution, we compare its fitness to the old solution through use of the fitness function, which should relate to the metric we are optimizing. For vertex coloring, a common way to assign a fitness value to a coloring is to count the number of colors used, which we then seek to minimize. Therefore, if the new solution uses fewer colors than the old one, it will take its place as the current candidate solution.

The most common local search algorithm is called Random Local Search (RLS). It works by choosing uniformly at random which part of the candidate solution (be it the color of a vertex or, in different examples, the value of a bit) to flip to a different value. It performs well on problems where the global optimum can be reached by continuously increasing the fitness of solutions, as shown by \citeauthor{selman_hill-climbing_2006} \cite{selman_hill-climbing_2006}. However, RLS cannot leave non-global local optima once it reaches them. In this, it differs from global search algorithms such as the well-known \ooea -- an algorithm that chooses independently for each part of the candidate solution whether to mutate it, making it possible to change the entire candidate solution in a single mutation --, which have the entire search space available to them in each mutation \cite{jansen_analyzing_2013}.

Like most heuristic algorithms, RLS is a \textit{black-box} algorithm \cite{jansen_analyzing_2013}. This means that it does not incorporate knowledge about the specific problem. When applying RLS to vertex coloring or another optimization problem such as the vertex cover or traveling salesman problem, the fitness function will be changed to reflect the new problem, but the general algorithm including the mutation operator will remain the same. 

The opposite approach is problem-specific algorithms, where an algorithm is written specifically to solve a single optimization problem, but cannot be used for other problems. In comparison, black-box algorithms have the advantage that they are widely applicable and require less work on the side of the software developer. However, this comes at the drawback that black-box algorithms often have worse run time \cite{jansen_analyzing_2013}.

A middle ground between the two approaches are \textit{gray-box} algorithms. They seek to combine the advantages of both strategies by incorporating some problem-specific knowledge into otherwise black-box algorithms in the hopes of improving the run time. We develop such a gray-box operator and analyze its performance in \Cref{ch_kmn}.

\section{Related Work}

Local search is a much-studied heuristic. Although it is rarely used on its own due to problems escaping non-global local optima, it forms the basis for many other algorithmic strategies. We first offer a general overview of landscape aspects that influence local search and give examples of local search based strategies that were created to better deal with the challenge of non-global local optima. Then, we go into more detail on gray-box operators as an improvement on pure black-box strategies.

\subsection{Local Search Landscape}

Pure local search algorithms are essentially \textit{hillclimbers}, meaning that they can continuously increase the fitness (climb up the hill) until an optimum (the hilltop) is reached, but from there, they can never climb down again to reach a higher peak. On some problems, this behavior is beneficial and leads to good results. This is shown by \citeauthor{selman_hill-climbing_2006} \cite{selman_hill-climbing_2006} for the $N$-queens problem and it is further argued that it works as well for the traveling-salesman problem.

However, this strategy becomes a detriment once a fitness landscape -- meaning the connection of neighboring solutions and their fitness values -- contains non-global local optima. Therefore, when deciding whether local search can solve a given problem, it makes sense to first analyze the landscape for the existence of non-global local optima, which is what we will focus on in \Cref{ch:landscape} for the vertex coloring problem. However, this analysis is only the start. Beyond the existence of non-global local optima, one can also analyze the effect of the size and shape of their \textit{basins of attraction}, meaning the regions from which local search will reach a certain optimum, or the number of non-global local optima. All these factors are relevant for how likely local search is to get stuck in a non-global local optimum instead of reaching the local optimum.

\citeauthor{rudolph_escaping_2022} \cite{rudolph_escaping_2022} analyze several variations of local search algorithms for how well they deal with various testing functions, all of whose fitness landscapes contain non-global local optima with differently-shaped basins of attraction. While the standard Random Local Search (RLS) is found to get stuck in non-global local optima with constant probability on all testing functions, other variations of local search perform better under certain circumstances. Their advantage over RLS is that they have a mechanism for broadening the search space or otherwise escaping non-global local optima. 

For example, Variable Neighborhood Search (VNS), as introduced by \citeauthor{hansen_variable_2018} \cite{hansen_variable_2018}, gradually broadens the neighborhood from solutions that differ in only one aspect (local search) to solutions that differ in two, three and more until an improvement is made or a maximum distance to the original solution is reached. \citeauthor{rudolph_escaping_2022} \cite{rudolph_escaping_2022} show that VNS performs well on several testing functions.

Another common strategy is Simulated Annealing. As explained by \citeauthor{jansen_analyzing_2013} \cite{jansen_analyzing_2013}, it chooses whether to accept new solutions with a probability that increases with the fitness improvement of the new solution and the time since the last new solution was accepted. This means that it can, at times, accept solutions that worsen the fitness and is thereby able to leave non-global local optima, improving its performance in situations where non-global local optima exist.

These strategies can be combined with \emph{restarts}, which means that multiple searches are performed, each from a different randomly generated starting point. As shown by \citeauthor{rudolph_escaping_2022}\cite{rudolph_escaping_2022}, this increases the probability of hitting at least one starting point from which a global optimum can be reached. Further, \citeauthor{reeves_statistical_2004} \cite{reeves_statistical_2004} show that the results of multiple restarts can be used to estimate the number of local optima and thereby the probability of having found the global optimum. They also suggest that the number of local optima are an important measure for how well the different search strategies will perform on certain problems.

In general, however, \citeauthor{rudolph_escaping_2022} \cite{rudolph_escaping_2022} conclude from testing different black-box strategies on various testing functions that no single strategy outperforms all others on all problems and therefore, it is beneficial to incorporate some problem knowledge in the choice of the algorithm. 

\subsection{Gray-box}

As black-box algorithms, RLS and similar strategies suffer from limitations on their run time. The \textit{No Free Lunch Theorem}, as given by \citeauthor{jansen_analyzing_2013} \cite{jansen_analyzing_2013}, states that, on average over all possible problems, all black-box algorithms have the same run time. As a brute force approach is also a black-box algorithm, this implies that all black-box algorithms will have exponential run time on average over all problems. While they can and do perform well on specific problems, this still gives a motivation for moving away from a strictly black-box approach and towards gray-box algorithms. 

There are already various instances where gray-box operators have been successfully used for solving difficult problems. For example, \citeauthor{peters_mixed_2019} \cite{peters_mixed_2019} develop a \ooea using custom mutation operators, including several gray-box operators, for solving a real-world staff assignment problem. This algorithm vastly outperforms an alternative approach using Mixed Integer Programming.

Further, gray-box knowledge of the problem frequently finds use in genetic algorithms using crossover, meaning algorithms where multiple old solutions are combined to create a new one. \citeauthor{friedrich_crossover_2023} \cite{friedrich_crossover_2023} suggest that any crossover operator should be either balanced, order-unbiased and inheritance-respectful, or incorporate problem knowledge, implying that problem knowledge is an essential part of crossover operators.

Another common use case for gray-box mutation operators is on $NP$-hard problems, where they often improve the run time when compared to their black-box counterparts, which will also be shown to be the case in this work. In the following, we provide some examples of gray-box operators being used on different $NP$-hard problems.

\citeauthor{whitley_mk_2015} \cite{whitley_mk_2015} uses gray-box operators for solving certain pseudo-Boolean optimization problems in polynomial time. Likewise, \citeauthor{baguley_analysis_2022} \cite{baguley_analysis_2022} develop a gray-box mutation operator for the \ooea on the vertex cover problem and compare its performance to the standard \ooea. Here, the gray-box operator is able to improve the run time from $\Theta(n^{4})$ to $\Theta(n^{3})$ on paths and drastically increases the probability of finding the global optimum on complete bipartite graphs.

For the vertex coloring problem, which we will further analyze in this work, \citeauthor{sudholt_crossover_2005} \cite{sudholt_crossover_2005} proves that a \ooea has an expected run time in $\exp(\Omega(n))$ on complete binary trees, which \citeauthor{heinen_exponential_2025} \cite{heinen_exponential_2025} proves to likewise hold for paths and complete bipartite graphs. The proof given can be adapted to prove that the same run time holds for RLS. This shows that standard black-box strategies take a very long time to find an optimal coloring, making them ill-suited to solving the vertex coloring problem.

An improved algorithm incorporating problem knowledge is introduced by \citeauthor{cheong_analysis_2010} \cite{cheong_analysis_2010}, who develop an iterated local search algorithm and introduce two mutation operators using different degrees of problem knowledge. The general principle of the algorithm is to first perform a mutation and then a local search, repeating this process until an optimum is found.

The algorithm uses a variant of local search called \emph{Grundy Local Search}, as introduced by \citeauthor{hedetniemi_linear_2003} \cite{hedetniemi_linear_2003}. Colors are assigned indices at the beginning and during local search, each vertex is flipped to the smallest color that does not appear in its neighborhood. This is similar to our operator insofar that we also work by assigning ranks to colors ranks. However, \citeauthor{hedetniemi_linear_2003} \cite{hedetniemi_linear_2003}, determine the indices in the beginning and have them remain fixed, while our operator calculates indices depending on the number of appearances of the color in the current coloring. Further, Grundy Local Search prefers colors of smaller indices by assigning each vertex the smallest color not taken by any of its neighbors until all colors have the smallest possible color, while we remove colors of smaller rank with higher priority.

The first mutation operator introduced utilizes \emph{Kempe chains}, the idea of which is to swap all instances of two colors in a component of the graph containing only these colors. It is shown that an iterated local search using Kempe chains as the mutation operator performs well on paths \cite{cheong_analysis_2010} and complete binary trees \cite{bossek_time_2021}, but takes exponential time on general bipartite graphs \cite{cheong_analysis_2010}.

As an improvement on the Kempe chains mutation operator, \citeauthor{cheong_analysis_2010} \cite{cheong_analysis_2010} introduce an alternative operator called \textit{color elimination}. Based on multiple Kempe chain operations, its goal is to remove a given color from the neighborhood of a vertex by replacing all instances with a different color. Similar to our work, this mutation operator makes use of the knowledge that it is beneficial to remove several instances of one color. An iterated local search using color elimination can $2$-color any bipartite graph in $\mathcal{O}(n^{2} \log n)$, with other polynomial-time results for odd rings, sparse random graphs and 5-coloring planar graphs \cite{cheong_analysis_2010}. 

This work is extended by \citeauthor{bossek_time_2021} \cite{bossek_time_2021}, who also find the color elimination operator, which uses more problem knowledge, to outperform Kempe chains on dynamic graph coloring problems. Further, the results show that algorithms tailored to the region of the graph where dynamic changes are made outperform more general algorithms, which again shows that incorporating problem knowledge is advantageous.

    \chapter{Preliminaries}
    
In this chapter, we introduce the concepts we will be using throughout the work. We begin with definitions for the vertex coloring problem. Afterwards, we formally define the Random Local Search algorithm, which will be used for comparison frequently, and introduce our fitness function. Finally, we give a brief overview over relevant theorems and concepts from probability theory.

\section{Vertex Coloring}

Intuitively, the idea behind vertex coloring is to assign each vertex in a given graph a color, with a coloring being called \emph{proper} if no two adjacent vertices share the same color.

The following definition is based on \citeauthor{balakrishnan_graph_2012} \cite{balakrishnan_graph_2012}. 
\begin{definition}[Vertex Coloring]
    Let $G = (V, E)$ be a graph and $S$ a set of distinct colors. We call a map $c \colon V \to S$ a \emph{vertex coloring} of $G$. For a vertex $v \in V$, we call $c(v)$ its color.
    If, for all $\{ u, v \} \in E$,  it holds that $c(u) \ne c(v)$, then we call $c$ a proper coloring.
\end{definition}

We further introduce the \emph{term monochromatic edge} to refer to edges between two vertices of the same color.

\begin{definition}[Monochromatic Edge]
    Let $G = (V, E)$ be a graph, $S$ a set of distinct colors and $c \colon V \to S$ a vertex coloring of $G$. Then, we call an edge $\{ u, v \} \in E$ with  $c(u) = c(v)$ a \emph{monochromatic edge}.
\end{definition}

Note that a coloring is proper if and only if it contains no monochromatic edges.

Finding a proper coloring for any given graph is trivially accomplished by assigning a unique color to each vertex. However, when talking about colorings, the minimum number of colors needed for a proper coloring is often of interest. Here, we call a graph \emph{$k$-colorable} if a proper coloring using $k$ colors exists. \citeauthor{balakrishnan_graph_2012} \cite{balakrishnan_graph_2012} define this as follows.

\begin{definition}[$k$-Coloring]
    Let $G = (V, E)$ be a graph, $S$ a set of distinct colors and $c \colon V \to S$ a vertex coloring of $G$. $c$ is a $k$-coloring for $G$ if $\vert S \vert = k$. $G$ is called \emph{$k$-colorable} if such a coloring exists and it is proper. 
\end{definition}

Using this definition, we define classes of graphs depending on whether they are $k$-colorable for a given $k$. For example, $2$-colorable graphs are called \emph{bipartite graphs}, consisting of two partitions that are each colored in a single color.

\begin{definition}[Bipartite Graphs]
    Let $G = (V, E)$ be a graph. We call $G$ a \emph{bipartite graph} if it is $2$-colorable. Let $c$ be such a proper $2$-coloring. A subset $V' \subset V$ is called a partition of $G$ if, for all $u, v \in V'$, it holds that $c(u)=c(v)$ and for all $w \notin V'$, it holds that $c(u) \ne c(w)$. 
\end{definition}

A special sub-type of bipartite graphs relevant to this work is \emph{complete bipartite graphs}, where each vertex is connected to all vertices in the opposite partition. 

\begin{definition}[Complete Bipartite Graphs]
    Let $G = (V, E)$ be a graph. We call $G$ a complete bipartite graph and denote it by $K_{M,N}$ if it is a \emph{complete bipartite graph} with partitions of size $M$ and $N$ and if any vertex in one partition is adjacent to all vertices of the other partition. We call the partitions $V_M, V_N \subset V$; this then means that for all $v \in V_M$ and $u \in V_N$, $\{u,v\}\in E$.
\end{definition}

\section{Local Search}

We now introduce the algorithmic concepts used in this work. 

A general trajectory based algorithm for finding a minimal coloring on a given graph is given in \Cref{alg:tboa}. The algorithm starts with the only coloring we know for sure to be proper, namely each vertex colored in a distinct color. Given a current coloring $x$, we perform a mutation with a given \textsc{Mutate}-operator on $x$, receiving a new coloring $x'$. If its fitness is as least as good as that of $x$, we then replace $x$ with $x'$. This process is repeated, so that we continuously improve our coloring while possible.

As given in the algorithm, this loop continues to infinity, as we can never know when no further improvements can be made. In practice, trajectory based optimization algorithms will have a stopping condition such as the number of iterations since the beginning of the algorithm or since the last improvement \cite{jansen_analyzing_2013}. However, for theoretical analysis, we are instead interested in the expected number of fitness evaluations until the global optimum is reached for the first time.

Let $n\in \mathbb{N}$ be the number of vertices in a graph $G$ and $f\colon \mathbb{N}^n \to \mathbb{N}^{n}$ be the fitness function. Given $\textsc{Mutate}\colon \mathbb{N}^n \to \mathbb{N}^n$ (the mutation operator), our generalized trajectory based optimization algorithm is then given as follows. Note that while we defined a coloring as a mapping $c$ of vertices to colors, while working with algorithms, we assume an order of vertices $v_1, \ldots, v_n$ and use an array $x \in [n]^n$ for colorings, where $x[i] = c(v_i)$ for all $i \in [n]$.

\begin{algorithm}
    let $x = (1, \dots, n)$\;
    \While{true}
    {
        let $x' = \textsc{Mutate}(x)$\;
        \If{$f(x') \leq f(x)$}
        {
            set $x \leftarrow x'$\;
        }
    }
	\caption{Trajectory Based Optimization Algorithm}
    \label{alg:tboa}
\end{algorithm}

We will now first introduce the fitness function we use and then give a common local search mutation operator.

\subsection{Fitness Function \textsc{NumUsedColors}} \label{sec:num_used_cols}

The general idea of our fitness function is to gradually decrease the number of colors used in the graph by assigning better fitness to colorings that use fewer colors. Better fitness in this case means a lower fitness value, as we seek to minimize the fitness. At the same time, we must ensure that our colorings remain proper, meaning that they contain no monochromatic edges.

The fitness function works on tuples using the following lexicographical ordering. For $n \in \mathbb{N}$ and $x, y \in \mathbb{N}^n$, we write $x < y$ if there exists an $i \in [n]$ such that $x[i] < y[i]$ and, for all $j < i$, $x[j] = y[j]$. That is, $x$ is less than $y$ in the first entry where they disagree. We write $x \leq y$ if $x = y$ or $x < y$. This ordering relation is trivially total.

Let $\textsc{NumUsedColors} \colon [n]^n \to \mathbb{N}^{2}$ be the following fitness function: For a given coloring $x \in [n]^n$, let $m$ be the number of monochromatic edges, which is given by  $\lvert \{ \{ u, v \} \in E \mid x[u] = x[v] \} \rvert$. Further, let $k \in \mathbb{N}$ be the total number of colors that appear in $x$. Then, let $\textsc{NumUsedColors}(x) = (m, k)$.

This fitness function's primary objective is minimizing the number of monochromatic edges in the graph. As a secondary objective, it minimizes the total number of colors used in the graph. Therefore, once we reach a proper coloring, the fitness function never accepts an improper coloring afterwards, but prefers smaller values of $k$. Note that since our algorithm starts in a proper coloring, no mutation which creates an improper coloring will ever be accepted. The first part of the tuple therefore functions as a penalty term, ensuring that our coloring never turns improper.

\subsection{Random Local Search (RLS)}

While there exist many mutation operators for trajectory based optimization algorithms, we allow only one color to change per mutation. As such, we only consider local search algorithms, which are algorithms where any offspring is in the direct neighborhood of the parent.

We introduce the operator One-Flip. On input $x \in [n]^n$, it chooses $j \in [n]$ and a value in $[n]$ to change the color in position $j$ to uniformly at random, then returns the result.
\begin{algorithm}
    \textbf{Input:} $x \in [n]^{n}$\;
    {
        choose $j \in [n]$ uniformly at random\;
        choose $c \in [n]$ uniformly at random\;
        $x[j]=c$\;
        \Return $x$\;
    }
    \caption{One-Flip}
    \label{alg:rls}
\end{algorithm}

Note that using One-Flip as the mutation operator in the trajectory based optimization algorithm (\Cref{alg:tboa}) creates the well-known Random Local Search algorithm.

\section{Probability Theory}

Unlike with deterministic algorithms, it is impossible to determine an exact run time for randomized algorithms. Therefore, we instead try to bound the expected value of their run time.

To do so, we make use of the concept of stochastic processes. 

\begin{definition}[Stochastic Process]
    A stochastic process is a family of random variables, indexed against another variable or a set of variables. We write $(X_t)_{t \in \mathbb{N}}$ for random variables $X$ and index $t$.
\end{definition}

When working with local search on vertex coloring, that family of random variables usually describes some property of the coloring, such as the number of occurrences of a certain color. We then index against the number of fitness evaluations we already performed, which we refer to as time. Note that the number of fitness evaluations is equal to the number of mutations performed in the given algorithms.

For example, we might have a stochastic process $(X_t)_{t \in \mathbb{N}}$ where each random variable $X_t$ denotes the number of vertices colored in color $1$ at time $t$. Taken individually, each of those variables are random. However, given a certain $X_t$, we know which values $X_{t+1}$ can take (namely $X_t$, $X_t-1$ or $X_t+1$, as we can only change one color per mutation) and what their probabilities are.

We use such stochastic processes to calculate the expected run time of our algorithms on a given problem by mapping the problem to a stochastic process and then determining how long it takes that process to reach $0$. This is accomplished through use of a class of theorems called drift theorems. The idea behind drift is to determine whether, for a given $X_t$, it is more likely for $X_{t+1}$ to be greater or less than $X_t$ and with what probability either of those changes will occur. 

We use the following drift theorem called Additive Drift Theorem, which is given as Theorem 2.1 by \citeauthor{kotzing_theory_2024} \cite{kotzing_theory_2024}. It describes a situation where there exists a value $\delta$ such that $X_t$ always decreases by at least $\delta$ in expectation. Then, we can conclude that the time $T$ until the process reaches $0$ is at most $\frac{\Ex{X_0}}{\delta}$, with $X_0$ being the starting value of the stochastic process.

\begin{theorem}[Additive Drift, Upper Bound] \label{thm:additive_drift}
    Let $(X_t)_{t \in \mathbb{N}}$ be an integrable process over $\mathbb{R}$, and let $T = \min\{t \in \mathbb{N} \mid X_t \leq 0\}$. Furthermore, suppose the following two conditions hold (non-negativity, drift).
    \begin{enumerate}
        \item[(NN)] For all $t \leq T, X_t \geq 0$.
        \item[(D)] There is a $\delta > 0$, such that, for all $t < T$, $\Ex{X_{t} - X_{t + 1} ~|~ X_0, \dots, X_t} \geq \delta$.
    \end{enumerate}
    Then \[
    \Ex{T} \leq \frac{\Ex{X_0}}{\delta}.
    \]
\end{theorem}

    \chapter{Landscape Analysis}
    \label{ch:landscape}

In this chapter, we analyze the local search fitness landscape of bipartite graphs using the fitness function \textsc{NumUsedColors} as defined in \Cref{sec:num_used_cols}. Note that the fitness is given in tuples of values, which we order lexicographically. 
For this, let $\chi=[n]^{n}$ be the set of all possible colorings. We define a neighborhood $$R=\{(x,y)\in \chi^{2} \mid \exists ! i \in [n]: x[i] \ne y[i]\}.$$ This means that we consider two colorings $x$ and $y$ to be neighbors if $x$ can be transformed into $y$ by changing the color of a single vertex. 

In order to better talk about how colorings can be transformed into each other, we introduce the term \emph{positive path}. Intuitively, it means that there is a path of neighboring colorings leading from $x$ to $y$ such that no coloring on the path has a worse fitness than its predecessor. 

\begin{definition} [Positive Path] \label{def:pp}
    Let $R$ be a neighborhood as above and let $x,y$ be colorings. Then, we say that a \emph{positive path} exists from $x$ to $y$ if there are colorings $z_1, \dots, z_i$ such that $(x, z_1), (z_1, z_2), \dots , (z_i, y) \in R$ and $\textsc{NumUsedColors}(x) \ge \textsc{NumUsedColors}(z_1) \ge \dots \ge \textsc{NumUsedColors}(z_i) \ge \textsc{NumUsedColors}(y)$. 
\end{definition}

We seek to analyze the landscape $R$ for the existence of local optima, meaning, intuitively, colorings where all neighboring colorings have a worse fitness. We are especially interested in non-global local optima, in this case meaning colorings using more than $2$ colors that are still local optima. We now give a formal definition for the different types of local optima.

\begin{definition} [Local Optimum]
    Let $x$ be a coloring and $R$ a neighborhood as defined above. Then, we call $x$ a \emph{local optimum} if, for all colorings $y$ with $(x,y) \in R$, it holds that $\textsc{NumUsedColors}(x) < \textsc{NumUsedColors}(y)$. Further, we call $x$ \emph{global optimum} if, for all colorings $z$, $\textsc{NumUsedColors}(z) \ge \textsc{NumUsedColors}(x)$. 
\end{definition}

Now we define the term \emph{plateau} to refer to a set of neighboring colorings such that all of them have the same fitness. 

\begin{definition} [Plateau]
    Let $P \subset \chi$ be a set of colorings such that for all $x, y \in P$ it holds that $\textsc{NumUsedColors}(x) = \textsc{NumUsedColors}(y)$ and that there exists a positive path between $x$ and $y$ in $P$. Then, we call $P$ a \emph{plateau}.
\end{definition}

We now define a local optimum consisting of a plateau.

\begin{definition} [Plateau Local Optimum]
    Let $P \subset \chi$ be a plateau. We call $P$ a \emph{plateau local optimum} if, for all colorings $y \notin P$ neighboring a coloring $x \in P$, it holds that $\textsc{NumUsedColors}(x) < \textsc{NumUsedColors}(y)$.
\end{definition}

Note that any local optimum is also a plateau local optimum. This means that in proving no non-global plateau local optima exist, we also show that no non-global local optima exist. 

Further, note the relationship between non-global local plateau optima and positive paths, which we will use in the following proofs: For any coloring $x$ that is part of a non-global plateau local optimum, no positive path exists to a global optimum. This follows from the fact that all colorings neighboring the plateau have a worse fitness than the plateau, making such a positive path impossible. 


We now analyze the neighborhood relationship defined above on different types of bipartite graphs for local optima in the landscape. First, we analyze graph configurations where no non-global local optima exist. Then, we provide graph configurations containing non-global local optima. For both, we regard only local optima consisting of proper colorings.

\section{Cases Without Non-global Local Optima}

In this section, we examine the following configurations of bipartite graphs where no non-global plateau local optima exist:
\begin{itemize}
    \item At least one vertex is universal to the opposite partition (see \Cref{def_universal}). 
    \item The graph is acyclic.
\end{itemize}

In the following, we analyze both of these configurations in detail.

\subsection{At Least One Vertex is Universal to the Opposite Partition} \label{sec:universal}

Intuitively, we refer to a vertex as \emph{universal to the opposite partition} if it is connected to all vertices in the opposite partition. One possible example for this are complete bipartite graphs, where every vertex is universal to the opposite partition.

\begin{definition} [Universal to the Opposite Partition] \label{def_universal}
    Let $G=(V, E)$ be a bipartite graph with partitions $V_M, V_N$. Without loss of generality, let $v \in V_N$. Then, we call $v$ \emph{universal to the opposite partition} if, for all $u \in V_M$, it holds that $\{u, v\} \in E$.
\end{definition}

Now, we prove that the existence of a vertex that is universal to the opposite partition implies that no non-global plateau local optima exist.

\begin{theorem} \label{the:universal}
   If there is at least one vertex universal to the opposite partition in a bipartite graph, no non-global plateau local optima in $R$ exist.
\end{theorem}
\begin{proof}

Let $G=(V,E)$ be a bipartite graph and let $V_N, V_M$ denote its partitions. Without loss of generality, let $v \in V_N$ be a vertex that is universal to the opposite partition, meaning that for all $u \in V_M$, there exists an edge $\{u, v\}\in E$. 

Let $x$ be a proper coloring that contains at least three colors. We now prove that $x$ is not in a plateau local optimum by showing the existence of a positive path from $x$ to a coloring containing only two colors. 

As $x$ is proper, we know that it contains no monochromatic edges. Therefore, for any given vertex $u \in V_M$, we have $x[u] \ne x[v]$ as we know that $\{u, v\}\in E$. This means the color $x[v]$ does not appear anywhere in the partition $V_M$.

We now prove the existence of a positive path in two steps: First, we show that a positive path exists from $x$ to a coloring $x_i$ where all vertices in $V_N$ are colored in the same color. Then, we show a positive path exists from $x_i$ to a coloring $y_l$, where all vertices in $V_M$ are colored likewise in one color. As $y_l$ then contains only 2 colors, it holds that $\textsc{NumUsedColors}(x) > \textsc{NumUsedColors}(y_l)$ and $y_l$ is a global optimum.

Assuming that $V_N$ is colored in at least two colors in $x$, let $v_1, \dots, v_i \in V_N$ be the vertices in $V_N$ such that for all $j\le i$, it holds that $x[v_j] \ne x[v]$. Then, let $x_1$ be a coloring neighboring $x$ such that $x_1[v_1] = x[v] \ne x[v_1]$. As no vertex in $V_M$ is colored in $x[v]$, we know that $x_1$ contains no monochromatic edges. Further, the number of colors in use did not increase, as the color $x_1[v_1]$ was already in use in $x$ as $x[v]$. Therefore, $\textsc{NumUsedColors}(x) \ge \textsc{NumUsedColor}(x_1)$.

For all $1<j\le i$, let $x_j$ be a coloring neighboring $x_{j-1}$ with $v_j$ colored such that $x_j[v_j] = x[v] \ne x_{j-1}[v_j]$. As no vertex in $V_M$ is colored in $x[v]$, we know that $x_j$ contains no monochromatic edges. Further, the number of colors in use did not increase as the color $x_j[v_j]$ was already in use in $x_{j-1}$ as $x_{j-1}[v]$. Therefore, $\textsc{NumUsedColors}(x_{j-1}) \ge \textsc{NumUsedColor}(x_j)$.

Using this, we know that a positive path $x, x_1, \dots, x_i$ exists from a coloring $x$ to a coloring $x_i$ where all vertices in $V_N$ are colored in the same color.

Now, we will prove the existence of a positive path from $x_i$ to a coloring $y_l$.

Let $c$ be a color in use in $V_M$ in $x_i$. We note that this color does not appear in $V_N$, which is now colored in a single color, as $x_i$ contains no monochromatic edges. Let $u_1, \dots, u_l \in V_M$ be vertices such that, for all $j \le l$, it holds that $x[u_j] \ne c$. Then, let $y_1$ be a coloring neighboring $x_i$ such that $y_1[u_1] = c \ne x_i[u_1]$. As no vertex in $V_N$ is colored in $c$, we know that $y_1$ contains no monochromatic edges. Further, the number of colors in the coloring did not increase, as the color $c$ was already in use in $x_i$ Therefore, $\textsc{NumUsedColors}(x_i) \ge \textsc{NumUsedColor}(y_1)$.

For all $1<j\le l$, let $y_j$ be a coloring neighboring $y_{j-1}$ such that $y_j[u_j] = c \ne y_{j-1}[u_j]$. As no vertex in $V_N$ is colored in $c$, we know that $y_j$ contains no monochromatic edges. Further, the number of colors in the coloring did not increase, as the color $y_j[u_j]$ was already in use in $y_{j-1}$. Therefore, $\textsc{NumUsedColors}(y_{j-1}) \ge \textsc{NumUsedColor}(y_j)$.

Using this, we know that a positive path $x_i, y_1, \dots, y_l$ exists from a coloring $x_i$ to a coloring $y_l$ where all vertices in $V_M$ are colored in the same color. 

As both $V_M$ and $V_N$ are colored in a single color in $y_l$, it is a $2$-coloring. With both of them proper, we know that $\textsc{NumUsedColors}(y_l) < \textsc{NumUsedColor}(x)$. Further, a positive path exists between $x$ and $y_l$, proving the theorem.

\end{proof}

\subsection{Acyclic Graphs}

In this section, we prove that acyclic graphs do not contain non-global local optima by showing that a positive path from any proper coloring to a proper $2$-coloring exists.

\begin{theorem} 
    Let $G=(V,E)$ be an acyclic graph. Then, $G$ does not contain plateau local optima at more than $2$ colors.
\end{theorem}

\begin{proof}
    As we work with undirected graphs, any acyclic graph is a forest. It therefore suffices to show that the theorem holds for trees, as a forest of multiple $2$-colored trees can be transformed into a proper $2$-coloring by gradually replacing all instances of unnecessary colors with either of two colors which will remain in the graph.

    Let $G=(V, E)$ be a tree with $n$ vertices and $x$ a proper coloring. We assume $G$ is rooted and, for any vertex $v$, refer with \textit{parent} to the vertex that follows $v$ on the path to the root and with \textit{child} to any other neighbor of $v$. We call a vertex a \textit{descendant} of $v$ if $v$ lies on the path from that vertex to the root. The \textit{sub-tree} created by $v$ refers to the tree containing $v$ and all of its descendants.
    
    We show that a positive path exists from $x$ to a proper $2$-coloring. To do so, we first prove that any vertex $v$ in $G$ can be re-colored to any color that already appears in $x$ outside of the sub-tree created by $v$, but is not the color of its parent.

    \begin{lemma}\label{lem:recolor}
        Let $G=(V, E)$ be a rooted tree, $x$ a coloring and $v \in V$ a vertex with parent $p$. Then, for all colors $c \ne x[p]$ that appear in $x$ as the color of any vertex outside of the sub-tree created by $v$, there exists a positive path to a coloring $x'$ such that $x'[v]=c$ and for all vertices $w$ that are not descendants of $v$, $x[w]=x'[w]$.
    \end{lemma}
    \begin{proof}
        We prove this by induction. Note that we refer with \textit{depth} of a vertex to the length of the longest path from that vertex to any leaf vertex descendant.

        Let $v \in V$ be a vertex of depth $0$, meaning a leaf vertex. We know that $v$ has only one neighbor in $G$, namely its parent. Let $c \ne x[p]$ be a color that appears in $x$ outside of $v$. Then, let $x'$ be a coloring such that $x'[v]=c$ and, for all vertices $u \ne v$, it holds that $x'[u]=x[u]$. This means that $x'$ is a proper coloring, since changing the color of $v$ to $c$ does not create monochromatic edges as $x'[p]=x[p]\ne c$. Further, the number of colors used in $x'$ is not greater than the number of colors used in $x$, as $c$ was already used in $x$. Therefore, $\textsc{NumUsedColors}(x') \le \textsc{NumUsedColors}(x)$ and a positive path exists from $x$ to $x'$ by flipping the color of $v$ to $c$.

        Now, assume that for any vertex $u$ of depth at most $i$ and coloring $x$, it holds that for all colors $c \ne x[p]$ that appear in $x$ outside of the sub-tree created by $v$, there exists a positive path towards a coloring $x'$ such that $x'[v]=c$ and for all vertices $w$ that are not descendants of $v$, $x'[w]=x[w]$. We now prove that the same holds for any vertex of depth $i+1$.

        Let $x$ be a proper coloring, $v$ a vertex of depth $i+1$ and $p$ its parent. We now construct a positive path to a coloring $x'$ such that $x'[v]=c$. We note that as $x$ is proper, $x[v] \ne x[p]$. Let $u_1,   \dots, u_l$ be the children of $v$. Let $c$ be a color that appears in $x$ outside of the sub-tree created by $v$ such that $x[p] \ne c$. 

        Let $x_1$ be a coloring such that $x_1[u_1] = x[p]$ and for all vertices $w$ that are not descended from $u_1$, it holds that $x_1[w]=x[w]$. We know that a positive path from $x$ to such a coloring exists as $u_1$ has a depth of at most $i$, $x[v] \ne x[p]$ and $x[p]$ appears outside of the sub-tree created by $u_1$. Let that path be called $P_1$.

        Let $x_2$ be a coloring such that $x_2[u_2] = x[p]$ and for all vertices $w$ that are not descended from $u_2$, it holds that $x_2[w]=x_1[w]$. We know that a positive path from $x_1$ to such a coloring exists as $u_2$ has a depth of at most $i$, $x[v] \ne x[p]$ and $x[p]$ appears outside of the sub-tree created by $u_2$. Let that path be called $P_2$.

        Likewise, we define all colorings $x_3, \dots x_l$ and paths $P_3, \dots, P_l$. Then, $x_l$ is a coloring with all children of $v$ colored in $x[p]$ and a positive path from $x$ to $x_l$ is given by $P_1, \dots, P_l$.

        Let $x'$ be a coloring such that $x'[v]=c$ and for all vertices $u \ne v$, it holds that $x'[u]=x_l[u]$. Then, $x'$ is proper as $v$ has no neighbors colored in $c$ in $x_l$ and $c$ already appears in $x_l$, so the number of colors does not increase. Therefore, it holds that $\textsc{NumUsedColors}(x') \le \textsc{NumUsedColors}(x_l)$ and a positive path exists from $x$ to $x'$, proving the lemma.
    \end{proof}

    Using this lemma, we can now prove that a positive path exists from any proper coloring on trees to a proper $2$-coloring by describing the order in which vertices are re-colored. In combination with the previous lemma, this then shows the existence of a positive path to such a $2$-coloring.

    Let $x$ be a proper coloring on $G$, $r$ the root of the tree and $v$ one of $r$'s children. Then, we can create a $2$-coloring by re-coloring all of $r$'s children, meaning the vertices of level $1$, to $x[v]$. $x[v]$ is a valid choice to re-color to as it appears as the color of $v$ (and therefore outside of the sub-tree created by any of the other children) and, due to $x$ being proper, is not the color of the parent $r$. We then re-color all vertices of level $2$ to $x[r]$ and so on. As \Cref{lem:recolor} ensures that colors outside of the sub-graph created by the vertex that is currently being re-colored remain unchanged, this strategy yields a positive path to a proper $2$-coloring.
\end{proof}

\section{Cases With Non-global Local Optima} \label{sec:reaching_local_opt}

In this section, we examine situations in which a non-global local optimum exists, first for arbitrary $k \le \frac{n}{2}$, then for three colors.

\subsection{Non-global Local Optima at $k$ Colors}

We first introduce a type of bipartite graph where both partitions have the same size, $k$, and each vertex has a degree of $k-1$. Depending on $k$, we call it a $k$-crown graph. 

\Cref{fig:k_n} shows two example instances of a $k$-crown graph for $k=3$ and $k=4$. Note that the $3$-crown graph is isomorphic to the $C_6$.

    \begin{figure} [hbt!] 
        \centering
        \begin{subfigure}[b]{0.45\textwidth}
        \captionsetup{justification=centering}
        \begin{tikzpicture}[
    greennode/.style={circle, draw=black!60, fill=white!5, very thick, minimum size=7mm},
    rednode/.style={circle, draw=black!60, fill=white!5, very thick, minimum size=7mm},
    bluenode/.style={circle, draw=black!60, fill=white!5, very thick, minimum size=7mm},
    ]
    \node[greennode]      (green1)                              {1};
    \node[greennode]        (greenA)       [below=2cm of green1] {A};
    \node[rednode]        (red2)       [right=of green1] {2};
    \node[rednode]        (redB)       [right=of greenA] {B};
    \node[bluenode]        (blue3)       [right=of red2] {3};
    \node[bluenode]        (blueC)       [right=of redB] {C};
    
    \draw[-] (green1.south) -- (redB.north);
    \draw[-] (green1.south) -- (blueC.north);
    \draw[-] (red2.south) -- (greenA.north);
    \draw[-] (red2.south) -- (blueC.north);
    \draw[-] (blue3.south) -- (greenA.north);
    \draw[-] (blue3.south) -- (redB.north);
    
    \end{tikzpicture}
        \caption{$k=3$}
        \end{subfigure}
            \begin{subfigure}[b]{0.45\textwidth}
            \captionsetup{justification=centering}
        \begin{tikzpicture}[
    greennode/.style={circle, draw=black!60, fill=white!5, very thick, minimum size=7mm},
    rednode/.style={circle, draw=black!60, fill=white!5, very thick, minimum size=7mm},
    bluenode/.style={circle, draw=black!60, fill=white!5, very thick, minimum size=7mm},
    whitenode/.style={circle, draw=black!60, fill=white!5, very thick, minimum size=7mm},
    ]
    \node[greennode]      (green1)                              {1};
    \node[greennode]        (greenA)       [below=2cm of green1] {A};
    \node[rednode]        (red2)       [right=of green1] {2};
    \node[rednode]        (redB)       [right=of greenA] {B};
    \node[bluenode]        (blue3)       [right=of red2] {3};
    \node[bluenode]        (blueC)       [right=of redB] {C};
    \node[whitenode]      (white4)            [right=of blue3] {4};
    \node[whitenode]        (whiteD)       [right=of blueC] {D};
    
    \draw[-] (green1.south) -- (redB.north);
    \draw[-] (green1.south) -- (blueC.north);
    \draw[-] (green1.south) -- (whiteD.north);
    \draw[-] (red2.south) -- (greenA.north);
    \draw[-] (red2.south) -- (blueC.north);
    \draw[-] (red2.south) -- (whiteD.north);
    \draw[-] (blue3.south) -- (greenA.north);
    \draw[-] (blue3.south) -- (redB.north);
    \draw[-] (blue3.south) -- (whiteD.north);
    \draw[-] (white4.south) -- (greenA.north);
    \draw[-] (white4.south) -- (redB.north);
    \draw[-] (white4.south) -- (blueC.north);
    
    \end{tikzpicture}
        \caption{$k=4$}
        \end{subfigure}
            \caption{Two example instances of a $k$-crown graph for $k=3$ and $k=4$.}
    \label{fig:k_n}
    \end{figure}
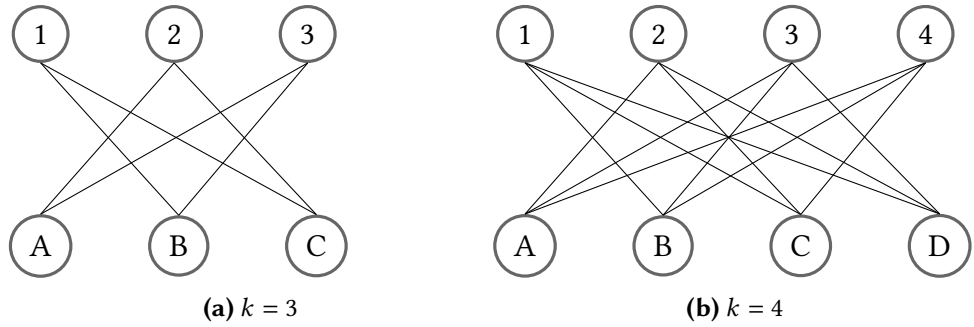

We further introduce the term \emph{opposing pair} to refer to two vertices in opposite partitions that are not connected by an edge.

\begin{definition}[Opposing Pair] \label{def:opp_pair}
    Let $G=(V,E)$ be a bipartite graph with partitions $V_M, V_N$. Let $u\in V_M$ and $v\in V_N$. Then, we call $u$ and $v$ an \emph{opposing pair} if $\{u, v\} \notin E$.
\end{definition}

We now prove that any bipartite graph $G = (V, E)$ contains a non-global plateau local optimum at $k$ colors with $k >2$, provided that it contains a $k$-crown graph as an induced subgraph and certain criteria regarding the structure of the rest of the graph are met. Intuitively, the $k$-crown graph contains a non-global local optimum (which will be proven in \Cref{lem:k_k}), which creates a non-global plateau local optimum in $G$ unless a part of $G$'s structure makes the necessary coloring of the $k$-crown graph impossible.

While the required structure is rather specific, note that this theorem proves that non-global local optima at up to $\frac{n}{2}$ colors (in the $\frac{n}{2}$-crown graph) are possible. This means that any local search algorithm using this neighborhood relationship can return arbitrarily bad results on bipartite graphs.

\begin{theorem} \label{the:local_opt}
    Let $k>2$. Let $G=(V, E)$ be a bipartite graph such that 
    \begin{enumerate}
    \item $G$ contains a $k$-crown graph as an induced subgraph,
    \item no vertex outside of the $k$-crown graph is connected to all vertices in one of the $k$-crown graph's partitions and 
    \item for all $u, v \in V$ and $a, b$ as vertices in the $k$-crown graph, it holds that if $(u,a), (v, b)$ and $(a,b)$ are opposing pairs, then so is $(u,v)$.
\end{enumerate}
    Then, a plateau non-global local optimum exists at $k$ colors. 
\end{theorem}

To prove this, we first show that the $k$-crown graph has a non-global local optimum at $k$ colors and then conclude that this causes a non-global local plateau optimum in $G$. We denote by $V_k \subset V$ the vertices in the $k$-crown graph.

\begin{lemma} \label{lem:k_k}
    There exists a non-global local optimum in the $k$-crown graph at $k$ colors.
 \end{lemma}
 \begin{proof}
     Let $x$ be a proper $k$-coloring. Then, $x$ is a local optimum if, for each vertex, it holds that all of its neighbors have distinct colors. \Cref{fig:local_opt} gives examples of such colorings for $k=3$ and the $k=4$.

    \begin{figure}[htb!]   
    \centering
    \begin{subfigure}[b]{0.45\textwidth}
    \captionsetup{justification=centering}
    \begin{tikzpicture}[
greennode/.style={circle, draw=green!60, fill=green!5, very thick, minimum size=7mm},
rednode/.style={rectangle, draw=red!60, fill=red!5, very thick, minimum size=6mm},
bluenode/.style={rectangle, draw=blue!60, fill=blue!5, very thick, minimum size=6mm, rounded corners},
]
\node[greennode]      (green1)                              {1};
\node[greennode]        (greenA)       [below=2cm of green1] {A};
\node[rednode]        (red2)       [right= 1.25cm of green1] {2};
\node[rednode]        (redB)       [right=1.2cm of greenA] {B};
\node[bluenode]        (blue3)       [right=1.25cm of red2] {3};
\node[bluenode]        (blueC)       [right=1.2cm of redB] {C};

\draw[-] (green1.south) -- (redB.north);
\draw[-] (green1.south) -- (blueC.north);
\draw[-] (red2.south) -- (greenA.north);
\draw[-] (red2.south) -- (blueC.north);
\draw[-] (blue3.south) -- (greenA.north);
\draw[-] (blue3.south) -- (redB.north);

\end{tikzpicture}
    \caption{$k=3$}
    \end{subfigure}
        \begin{subfigure}[b]{0.45\textwidth}
        \captionsetup{justification=centering}
    \begin{tikzpicture}[
greennode/.style={circle, draw=green!60, fill=green!5, very thick, minimum size=7mm},
rednode/.style={rectangle, draw=red!60, fill=red!5, very thick, minimum size=6mm},
bluenode/.style={rectangle, draw=blue!60, fill=blue!5, very thick, minimum size=6mm, rounded corners},
whitenode/.style={circle, draw=black!60, fill=white!5, very thick, minimum size=7mm},
]
\node[greennode]      (green1)                              {1};
\node[greennode]        (greenA)       [below=2cm of green1] {A};
\node[rednode]        (red2)       [right=of green1] {2};
\node[rednode]        (redB)       [right=of greenA] {B};
\node[bluenode]        (blue3)       [right=of red2] {3};
\node[bluenode]        (blueC)       [right=of redB] {C};
\node[whitenode]      (white4)            [right=of blue3] {4};
\node[whitenode]        (whiteD)       [right=0.875cm of blueC] {D};

\draw[-] (green1.south) -- (redB.north);
\draw[-] (green1.south) -- (blueC.north);
\draw[-] (green1.south) -- (whiteD.north);
\draw[-] (red2.south) -- (greenA.north);
\draw[-] (red2.south) -- (blueC.north);
\draw[-] (red2.south) -- (whiteD.north);
\draw[-] (blue3.south) -- (greenA.north);
\draw[-] (blue3.south) -- (redB.north);
\draw[-] (blue3.south) -- (whiteD.north);
\draw[-] (white4.south) -- (greenA.north);
\draw[-] (white4.south) -- (redB.north);
\draw[-] (white4.south) -- (blueC.north);

\end{tikzpicture}
    \caption{$k=4$}
    \end{subfigure}
        \caption{Two example instances of a $k$-crown graph for $k=3$ and $k=4$. As any given vertex is connected to at least one instance of each other color, any flip would lead to an improper coloring.}
\label{fig:local_opt}
\end{figure}
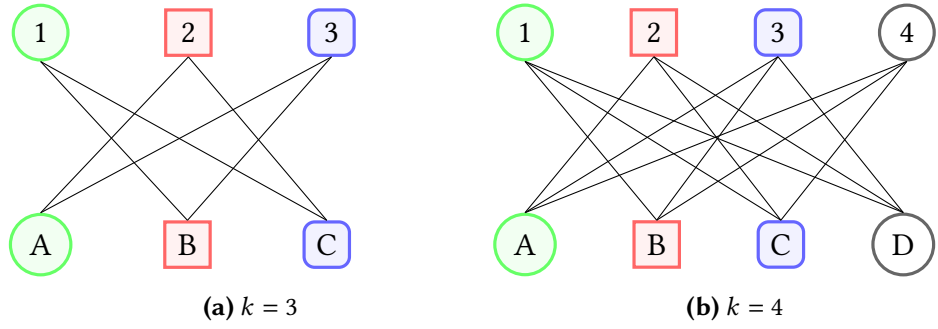

Assume $x$ is such a coloring, meaning that each of the $k$ colors appears exactly once in each partition. Let $v\in V_k$. We know that $v$ has $k-1$ neighbors and therefore $k-1$ neighboring colors. This means that, when changing the color of $v$, there are two possibilities for neighboring colorings $y$:

\begin{enumerate}
    \item $y[v]$ is one of the $k-1$ colors appearing in $x$ other than $x[v]$. That color must also be one of the $k-1$ colors neighboring $v$. Then, we gain one monochromatic edge, worsening the fitness in comparison to $x$.
    \item $y[v]$ is a color that does not appear in $x$. In that case, the number of colors in use in $y$ is greater than that in $x$, worsening the fitness.
\end{enumerate}

Therefore, any possible neighbors of $x$ have worse fitness, making $x$ a local optimum. We also know that $x$ is not a global optimum, as we have $k>2$ colors in use, but know the $k$-crown graph to be $2$-colorable.
     
 \end{proof}

 Now, we use this lemma to prove \Cref{the:local_opt}.
 \begin{proof} [Proof of \Cref{the:local_opt}]
     From \Cref{lem:k_k}, we know that we can $k$-color the $k$-crown graph in such a way that the coloring is a non-global local optimum by coloring each of the $k$ vertices in each partition in a distinct color (while maintaining a proper coloring). It remains to prove that such a coloring is always possible in $G$. 

     Now, assume that there are $k$ colors in use in $x$. Further, the $k$-crown graph is colored with each color used only once per partition. Let $x$ be a such coloring. Now, we must that for all vertices $v$ outside of the $k$-crown graph, there exists a color $x[v]$ such that $x$ remains proper.

     Let $v \in V \backslash V_k$. Let $a \in V_k$ be a vertex in the opposite partition as $v$ such that $(a,v)$ is an opposing pair. We know that such a vertex must exist, as no vertex in $G$ is connected to all vertices in one partition of the $k$-crown graph. Then, let $x[v]=x[a]$. 
     
     Now, let $u\in V$ be a vertex such that $\{u,v\} \in E$ (meaning $(u,v)$ is not an opposing pair) and let $b \in V_k$ such that $(u,b)$ is an opposing pair. Then, we know that $(a,b)$ cannot be an opposing pair, meaning  $\{a,b\} \in E$ must hold and therefore $x[a] \ne x[b]$, which also implies $x[v] \ne x[u]$. 

     This means that for any vertex $v$, there exists a color such that no monochromatic edge is created. Therefore, we can construct the described coloring without creating monochromatic edges, proving that a non-global plateau local optimum at $k$ colors exists.
 \end{proof}

\subsection{Non-global Local Optima at Three Colors}

Now, we examine a more general situation in which a non-global local optimum exists at three colors. For this, we introduce a structure we call $3$-circle as well as related concepts, which we then use to define a graph configuration that contains non-global local optima at three colors.

First, we define the term \emph{$3$-circle}.

\begin{definition}[3-circle]
    For a graph $G=(V,E)$, we call a subgraph a \emph{$3$-circle} if this subgraph is isomorphic  to $C_{3c}$ for some $c \in \mathbb{N}$, so a circle with a number of vertices which is divisible by $3$. We refer to the vertices within the circle as $v_1, \dots, v_{3c}$. 
\end{definition}

Now, we define a graph as \emph{$3$-circle-optimal}. Examples to clarify the definition are given in \Cref{fig:3_circle_opt}. Further, note that one example of a $3$-circle optimal graph is the $3$-crown graph as introduced in the previous subsection.

\begin{definition} [3-circle-optimal graph] \label{def:3co}
    Let $G=(V, E)$ be a graph. We call $G$ \emph{$3$-circle-optimal} if and only if a $3$-circle $v_1, v_2, \dots v_i \in V$ exists such that, for all $j \le i$, there exists no $l \le i$ such that $l \equiv j \pmod 3$ and $\{v_j, v_l \}\in E$.
\end{definition}

    \begin{figure}[htb!]   
    \centering
    \begin{subfigure}[b]{0.45\textwidth}
    \begin{tikzpicture}[
greennode/.style={circle, draw=green!60, fill=green!5, very thick, minimum size=7mm},
rednode/.style={rectangle, draw=red!60, fill=red!5, very thick, minimum size=6mm},
bluenode/.style={rectangle, draw=blue!60, fill=blue!5, very thick, minimum size=6mm, rounded corners},
]
\node[greennode]      (green1)                              {g1};
\node[rednode]        (red1)       [below=2cm of green1] {r1};
\node[bluenode]        (blue2)       [right=of green1] {b2};
\node[rednode]        (red2)       [right=of blue2] {r2};
\node[greennode]        (green2)       [below=2cm of red2] {g2};
\node[bluenode]        (blue1)       [left=of green2] {b1};

\draw[-] (green1.south) -- (red1.north);
\draw[-] (green1.east) -- (blue2.west);
\draw[-] (blue2.east) -- (red2.west);
\draw[-] (red1.east) -- (blue1.west);
\draw[-] (blue1.east) -- (green2.west);
\draw[-] (red2.south) -- (green2.north); 
\draw[-] (blue2.south) -- (green2.north);

\end{tikzpicture}
    \caption{A 3-circle-optimal graph containing two 3-circles: $(b2, g2, r2)$ and $(g1, r1, b1, g2, r2, b2)$.}
    \end{subfigure}
        \begin{subfigure}[b]{0.45\textwidth}
    \begin{tikzpicture}[
greennode/.style={circle, draw=green!60, fill=green!5, very thick, minimum size=7mm},
rednode/.style={rectangle, draw=red!60, fill=red!5, very thick, minimum size=6mm},
bluenode/.style={rectangle, draw=blue!60, fill=blue!5, very thick, minimum size=6mm, rounded corners},
whitenode/.style={circle, draw=black!60, fill=white!5, very thick, minimum size=7mm},
]
\node[greennode]      (green1)                              {g1};
\node[rednode]        (red1)       [below=2cm of green1] {r1};
\node[bluenode]        (blue2)       [right=of green1] {b2};
\node[rednode]        (red2)       [right=of blue2] {r2};
\node[greennode]        (green2)       [below=2cm of red2] {g2};
\node[bluenode]        (blue1)       [below=2.15cm of blue2] {b1};

\draw[-] (green1.south) -- (red1.north);
\draw[-] (green1.east) -- (blue2.west);
\draw[-] (blue2.east) -- (red2.west);
\draw[-] (red1.east) -- (blue1.west);
\draw[-] (blue1.east) -- (green2.west);
\draw[-] (red2.south) -- (green2.north); 
\draw[-] (blue2.south) -- (blue1.north);

\end{tikzpicture}
    \caption{A graph containing the $3$-circle $(g1, r1, b1, g2, r2, b2)$. It is not $3$-circle-optimal due to the edge $\{b1, b2\}$.}
    \end{subfigure}
        \caption{Two examples of 3-circles. One is also a 3-circle-optimal graph, the other is not. Vertices are colored in red, green and blue to better illustrate the positions within the circle. Intuitively, a graph is $3$-circle optimal if it contains at least one $3$-circle where no two vertices colored in the same color are connected by an edge.}
\label{fig:3_circle_opt}
\end{figure}
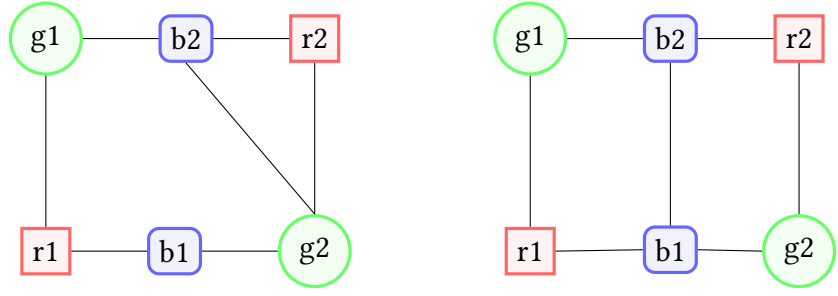
Further, we define the term \emph{$d$-opposing vertices} to describe a way vertices outside of the $3$-circle can be connected to vertices within.

\begin{definition} [$d$-opposing vertices]
   Let $G=(V, E)$ be a graph and let $v_1, \dots v_i \in V$ be a $3$-circle. Let $v \in V$ be a vertex that does not appear in the $3$-circle and let $d \in \{ 0, 1, 2 \}$. We call $v$ \emph{$d$-opposing} if, for all $j\le i$ such that $j \equiv d \pmod 3$, no edge exists between $v$ and $v_j$.
\end{definition}

Now, we prove that a bipartite graph contains a non-global plateau local optimum at three colors provided that it is $3$-circle optimal and certain other criteria for the other vertices are met, largely analogous to the logic used in the previous subsection.

\begin{theorem}
    Let $G=(V,E)$ be a bipartite graph such that 
    \begin{enumerate}
    \item it is 3-circle-optimal,
    \item for all $v \in V$ outside of the $3$-circle, there exists a $d \in \{0,1,2 \}$ such that $v$ is $d$-opposing and
    \item if two vertices $u, v \in V$ are both $d$-opposing, then $\{u,v\} \notin E$.
\end{enumerate}
Then, $G$ contains a non-global plateau local optimum at $3$ colors.
\end{theorem}
\begin{proof}

To prove this, we first construct a proper $3$-coloring $x$ and then show that it is a non-global plateau local optimum. 

First, let $v_1, v_2, \dots v_i \in V$ be the $3$-circle in $G$ for which the above criteria are met. For all $j\le i$, let $x[v_j] \equiv j \pmod 3$. Such a coloring is possible as each vertex appears only once, meaning we assign only one color per vertex. 

Now, we show that the coloring is proper. First, note that $i \equiv 0 \pmod 3$, so the first and last vertex are not colored in the same color (as the first vertex is colored in color $1$), and for all $j\le i$, it holds that $j \nequiv j+1 \pmod 3$. This means that no two adjacent vertices in the circle are colored in the same color. Further, we know that for all $j \le i$, there exists no $l \le i$ such that $l \equiv i \pmod 3$ and $\{v_j, v_l \}\in E$, which implies that no two vertices of the same color are connected by an edge, meaning the coloring is proper.

Now, it remains to be shown that we can color the remaining vertices in such a way that the coloring remains proper. For this, let $v\in V$ be a non-3-circle vertex. Let $d \in \{0,1,2\}$ such that $v$ is $d$-opposing, which we know must be the case for at least one value $d$. Then, let $x[v]=x[v_{d+1}]$. As $v$ is $d$-opposing, we know that it is not connected to any same-colored vertex within the $3$-circle.

Let $u$ be a vertex outside of the $3$-circle such that $x[v]=x[u]$. We know that, in order for $u$ to be colored in $x[v]$, it must be $d$-opposing. Then, it follows from the condition that no edge can exist between $v$ and $u$, making the coloring proper.

Now, we show that $x$ is a non-global plateau local optimum. For this, we show that the coloring of the $3$-circle is a non-global local optimum and that changes to the coloring in the rest of the graph have no bearing on the fitness as long as propriety is maintained, making it a plateau.

We note that $x$ has three colors in use, meaning that 
\begin{enumerate}
    \item it is a non-optimal coloring, as a coloring with two colors exists and
    \item any neighboring coloring using more colors would immediately have a worse fitness, meaning that we can focus only on neighboring colorings using three or fewer colors.
\end{enumerate}

First, we consider the coloring within the circle. Let $j \le i$ and $y$ be a coloring neighboring $x$ such that $x[v_j] \ne y[v_j]$. We know that both vertices neighboring $v_j$ (meaning $v_{j-1}$ and $v_{j+1}$) have different colors in $x$ and therefore also in $y$, as the color of $v_j$ is the only one that changed. Therefore, assuming the total number of colors does not increase, it must either hold that $x[v_j] = x[v_{j-1}]$ or $x[v_j] = x[v_{j+1}]$. Either of these options would lead to an improper coloring, thereby worsening the fitness. With this, it is shown that the coloring within the $3$-circle is a non-global local optimum.

As we now know that any change to a color within the $3$-circle worsens the fitness of the graph and there are three colors in use within the $3$-circle, we also know that we cannot decrease the number of colors used below three. Since the coloring is proper already, the only way to improve the fitness would be to reduce the number of colors used, meaning that no neighboring coloring where a vertex outside of the $3$-circle changes colors can improve the fitness. This proves that the local optimum is a plateau, as there might be changes outside of the $3$-circle that do not worsen the fitness, but none that will ever improve it.

\end{proof}

    \chapter{Gray-box Local Search on Complete Bipartite Graphs}
        \label{ch_kmn}

\begin{wrapfigure}{O}{0.4\textwidth}
   \centering
   \resizebox{0.7\linewidth}{!}{
   
            \begin{tikzpicture}
           [ greennode/.style={circle, draw=green!60, fill=green!5, very thick, minimum size=7mm},
rednode/.style={rectangle, draw=red!60, fill=red!5, very thick, minimum size=6mm},
bluenode/.style={rectangle, draw=blue!60, fill=blue!5, very thick, minimum size=6mm, rounded corners},
]
                \node[rednode] (L1) at (0, 8) {R};
                \node[rednode] (L2) at (0, 6) {R};
                \node[rednode] (L3) at (0, 4) {R};
                \node[greennode] (L4) at (0, 2) {G};
                \node[greennode] (L5) at (0, 0) {G};
                \node[bluenode] (R1) at (3, 7) {B}
                    edge [-] (L1) edge [-] (L2) edge [-] (L3) edge [-] (L4) edge [-] (L5);
                \node[bluenode] (R2) at (3, 5) {B}
                    edge [-] (L1) edge [-] (L2) edge [-] (L3) edge [-] (L4) edge [-] (L5);
                \node[bluenode] (R3) at (3, 3) {B}
                    edge [-] (L1) edge [-] (L2) edge [-] (L3) edge [-] (L4) edge [-] (L5);
                \node[bluenode] (R4) at (3, 1) {B}
                    edge [-] (L1) edge [-] (L2) edge [-] (L3) edge [-] (L4) edge [-] (L5);
            \end{tikzpicture}}
            \caption{$K_{5, 4}$ in the middle of the algorithm. Green is the least frequent color, so removing it is the fastest way to an optimal coloring.}
            \label{fig:graybox_reasoning}
        \end{wrapfigure}
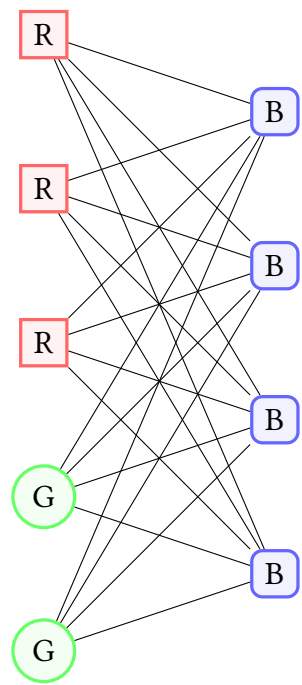

In the previous chapter we showed that many bipartite graphs contain non-global local optima. However, there are also configurations with only global optima, enabling local search to eventually find the optimum. One such class of graphs are complete bipartite graphs, as shown in \Cref{sec:universal}.

However, \citeauthor{heinen_exponential_2025} \cite{heinen_exponential_2025} shows that the probability of finding an optimal coloring on a complete bipartite graph in less than $2^{cn}$ iterations (for a constant $c$) lies in $2^{-\Omega(n)}$ for the \ooea, which can be easily translated to Random Local Search (RLS). This exponential run time is caused by the fact that, for less frequent colors, we are more likely to create new instances of those colors than to remove them, causing so-called \emph{negative drift}.

In this chapter, we introduce a gray-box local search mutation operator that removes less frequent colors with higher probability. We analyze its run time on complete bipartite graphs and prove that it takes $\Theta(n \log n)$ fitness function evaluations in expectation to find a proper $2$-coloring.

\section{The Gray-box Operator}

To improve on the performance of RLS, we introduce a gray-box mutation operator to be used as the mutation operator in \Cref{alg:tboa}. We call it \emph{gray-box local search}.

The mutation operator is based on the observation that it is beneficial to remove less frequent colors with higher probability. Take the coloring given in \Cref{fig:graybox_reasoning}. Here, green is the least frequent color currently in use in the graph. Therefore, the fastest way to reach an optimal coloring, meaning the one that requires the fewest mutations, is to flip both green vertices to red. 

The One-Flip operator used in RLS and given in \Cref{alg:rls} chooses a vertex to re-color uniformly at random, meaning that in \Cref{fig:graybox_reasoning}, we only have a probability of $\frac{2}{9}$ of choosing one of the green vertices. This is lower than the probability of choosing one of the red vertices, which is $\frac{1}{3}$, making it more likely to create a new green vertex than to remove one. Instead, in our gray-box mutation operator, we first choose a color to remove, with less frequent colors chosen exponentially more frequently. Then, we pick a vertex colored in this color uniformly at random and re-color it. 

To better express the frequency of colors in the graph, we introduce the term \emph{rank}.

\begin{definition}[Rank] \label{def:rank}
    Let $G = (V, E)$ be a graph and $x \in [n]^n$ a coloring of $G$. Let $a$ be an ordering of all colors in $x$ that appear at least once, ordered ascending by frequency, with ties broken by ordering the colors ascending by their identifiers. For a color $c \in [n]$, we call its $1$-based index in $a$ its \emph{rank}, denoted by $r_c$.
\end{definition}

For example, in \Cref{fig:graybox_reasoning}, the color of rank $1$, meaning the least frequent color, is green, followed by red with rank $2$ and blue with rank $3$.

The gray-box mutation operator is given in \Cref{alg:graybox_mut}. It first chooses the rank of the color that is meant to be reduced by polling from a geometric distribution with $p=\frac{1}{2}$. For a given color $c$, the probability of it being chosen equals $\frac{1}{2^{r_c}}$. In the example in \Cref{fig:graybox_reasoning}, this means that the probability of flipping a green vertex to a different color is $\frac{1}{2}$, $\frac{1}{4}$ for red and $\frac{1}{8}$ for blue. There is also a small probability ($\frac{1}{8}$ in the example) of no mutation being performed if the rank chosen from the distribution does not appear in the coloring.

Afterwards, a flip operation is performed, which is defined in the following. It accepts a coloring $x \in [n]^{n}$ and a color $c$ as parameters.

\textbf{ $\textsc{Flip}(x, c)$:} Choose a vertex $v$ colored in $c$ uniformly at random. Then choose a different color currently present in the coloring that is not adjacent to $v$ and re-color $v$ to this color. %

Here, we make use of the fact that we know any improper colorings or colorings using more colors than $x$ will be rejected by the fitness function. We therefore avoid unnecessary mutations by considering only colors that appear in $x$ and do not create monochromatic edges.

\SetKwBlock{WProbOneTwo}{with probability $\frac{1}{2}$ do}{else}

\begin{algorithm}[hbt!]
    \textbf{Input:} $x \in [n]^n$\;
    choose $r \sim \Geo \left( 1 / 2 \right)$\;
    let $k$ be the total number of colors used in $x$\;
    \If{$r > k$}
    {
        \Return $x$\;
    }
    let $c$ be the color of rank $r$\;

        \Return $\textsc{Flip}(x, c)$\;
    
    \caption{Gray-Box Mutation Operator}
    \label{alg:graybox_mut}
\end{algorithm}

\subsection{Run Time}

When analyzing the asymptotical run time of trajectory based algorithms, it is typical to take into account only the number of fitness evaluations. However, when comparing mutation operators, we believe that it is relevant to also consider the run time of the mutation operator. 

Let $G=(V,E)$ be a graph with $n$ vertices given as an adjacency list. Further, let $x \in [n]^{n}$ be the coloring. 

Over the run of the algorithm, we store two variables: Let $k$ be the total number of colors used in the current coloring, initialized with $k=n$. Further, let $R$ be an array such that, for a color $i$, $R[i]$ gives the number of appearances of $i$ in $x$. It is initialized with $1$ for all values. We note that we can use $R$ to find a color of rank $r$ (which is the $(r+n-k)$-th smallest element in $R$) in $\mathcal{O}(n)$ with the Introselect algorithm \cite{musserIntrospectiveSortingSelection1997}.

These variables are updated as follows when $x$ is mutated to $x'$: We decrement $R[x[v]]$ by one, as one appearance of that color in the graph disappeared and increment $R[x'[v]]$ by one, as a new instance of that color appeared. If a color is reduced to zero in such a way, we decrease $k$ by one and if a color is increased from zero, we increase $k$ as those cases mean that a color disappeared or appeared in $x'$. All of these operations happen in $\mathcal{O}(1)$.

Now, when performing a gray-box mutation, we randomly choose a rank from a geometric distribution. As explained above, we can find the color belonging to the chosen rank in $\mathcal{O}(n)$. For determining which colors are available to flip to, we create a copy of $R$ called $R'$, which happens in $\mathcal{O}(n)$. For all vertices $u$ neighboring the chosen vertex $v$, we set $R'[x[u]]=0$, which again takes $\mathcal{O}(n)$, as there are at most $n$ vertices adjacent to $v$. Now, any colors whose value in $R'$ is not $0$ are available to choose from for the flip. We can choose one uniformly at random and re-color $v$ to that color. 

Therefore, the total run time of the gray-box mutation operator lies in $\mathcal{O}(n)$. We argue that this is an acceptable run time for a mutation operator. While operators such as One-Flip may have a faster run time, any global operator has a run time in $\mathcal{O}(n)$ as well as it must decide for each vertex individually whether to flip it.

\section{Upper Bound} \label{sec:graybox_cbg}

We begin our analysis with an upper bound of $\mathcal{O}(n \log n)$ fitness function evaluations on the run time of gray-box local search using the fitness function $\textsc{NumUsedColors}$ on complete bipartite graphs with $n$ vertices.

\begin{theorem} \label{the:gb_on_knm}
   Gray-box local search using the fitness function $\textsc{NumUsedColors}$ finds an optimal coloring on the $K_{M,N}$ with $n$ vertices in $\mathcal{O}(n \log n)$.
\end{theorem} 
\begin{proof}
We use $k$ to denote the total number of colors currently in use in the graph and remove one color after another, decreasing ${k}$ on each fitness level. First, we examine the time it takes to remove the color of lowest rank (meaning the rarest color) that can still be reduced, depending on $k$. We refer to a color as being able to be reduced if a positive path [\Cref{def:pp}] from the current coloring to a coloring with no instance of that color exists. We know from \Cref{the:universal} that such a color must exist as long as we have more than two colors in use, as there must be a positive path to the global optimum containing two colors. This time then corresponds to the time to move from one fitness level to the next.

We introduce a stochastic process $(X_t)_{t \in \mathbb{N}}$ which, for any given time, denotes the number of vertices of the color of lowest rank that can still be reduced. If the number of colors in use in the coloring drops below $k$, we set $X_t=0$. Therefore, $X_t$ reaching $0$ corresponds to a color being removed from the graph. Let $\col_t$ denote the color whose appearances $X_t$ counts at time $t$. Note the dependence on $t$, as $X_t$ might refer to different colors over time.

Next we bound the time we spend on each fitness level. 

\begin{lemma} \label{lem:ct_to_0}
$X_t$ will reach $0$ in at most $8 \cdot \frac{n}{{k}}$ steps in expectation. 
\end{lemma}

To prove this, we seek to use \Cref{thm:additive_drift} [\nameref{thm:additive_drift}]. First, however, we prove that the probability of reducing $X_t$ is always greater than the probability of increasing it.

\begin{lemma} \label{lem:ct_reduced}
    For all $t \in \mathbb{N}$, it holds that
    \begin{align*}
        \Prob{X_{t + 1} < X_t \mid X_0, \dots, X_t} - \Prob{X_{t + 1} > X_t \mid X_0, \dots, X_t} \ge \frac{1}{8}.
    \end{align*}
\end{lemma} 
\begin{proof}
    
To prove this, let $k'$ denote the number of colors in use in the same vertex partition of the $K_{M,N}$ as $\col_t$. We introduce two events: The event $A$ denotes another color in the same partition as $\col_t$ being picked when choosing which color to select a vertex from for mutation. This corresponds to mutating a vertex in the same partition as $\col_t$ that is not itself colored in $\col_t$. The event $B$ denotes a vertex not previously colored in $col_t$ being flipped to $\col_t$ during mutation. 

In the following, we will use the definition of rank given in \Cref{def:rank} and denote the rank of a color $c$ with $r_c$. 

We then differentiate between two possible cases: First, we will consider a situation where $\col_t$ is the least frequent color currently in use in the graph. In the second case, we will analyze a situation where the opposite partition is already colored in a single color, which appears less frequently than $\col_t$, making $\col_t$ the second least frequent color as we work on a bipartite graph. We note that any other case (i.e. $\col_t$ having a rank greater than two) is impossible, as that would require at least three partitions. 
In each of the two cases, we will differentiate again between situations where there are only two colors left in the partition $\col_t$ appears in and situations where there are more than two colors. 

\textbf{Case 1:} $\col_t$ is the least frequent color in total, so $r_{\col_t}=1$.

    Now
    we get
    \begin{align*}
      \Prob{X_t - X_{t+1} >0 \mid X_0, \dots, X_t} = \frac{1}{2^{r_{\col_t}}}=\frac{1}{2}.
    \end{align*}
    
    On the other hand, the probability of $X_t$ increasing is
    \begin{align*}
    \Prob{X_t - X_{t+1} <0 \mid X_0, \dots X_t} 
    = \Prob{A} \cdot \Prob{B \mid A},
    \end{align*}
    as it corresponds to the probability of first choosing a vertex in the same partition as $\col_t$ but not colored in $\col_t$ (event $A$) and then flipping it to $\col_t$ (event $B$).
    
    We now differentiate between two further cases:
    
        \textbf{Case 1a:} There is only one other color in the same partition as $\col_t$, so $k' = 2$.
        
        It follows that $\Prob{A} \le \frac{1}{4}$, as the chance of picking that single other color is at most $\frac{1}{4}$ (since its rank must be at least $2$); and $\Prob{B \mid A} =1$, as the only color left in the graph that we can flip to (which must be a color in the same partition to avoid improper colorings) is $\col_t$. Therefore,
        $$
        \Prob{X_t - X_{t+1} < 0\mid X_0, \dots, X_t} \le \frac{1}{4}.
        $$
        
        \textbf{Case 1b:} There are more than $2$ colors in the same partition as $\col_t$, so $k' > 2$.

        Picking a vertex not colored in $\col_t$ in the same partition as $\col_t$ (which is the event A) is a prerequisite for $X_t$ increasing. It follows that $\Prob{A}$ must be smaller than the probability of $X_t$ not decreasing, meaning
        $$
        \Prob{A} < 1- \Prob{X_t - X_{t+1} >0 \mid X_0, \dots, X_t} = \frac{1}{2}.
        $$
        Further, the probability of flipping a vertex's color to $\col_t$, given that we already chose a vertex not colored in $\col_t$ in the same partition, equals $\frac{1}{k'-1}$. With $k' > 2$, we get
        \begin{align*}
          \Prob{B \mid A} =\frac{1}{k'-1} \le \frac{1}{2}.  
        \end{align*}
        Combining the two terms yields
        \begin{align*}
            \Prob{X_t - X_{t+1} <0 \mid X_0, \dots, X_t} \le \frac{1}{2} \cdot \frac{1}{2} = \frac{1}{4}.
        \end{align*}

    Therefore, in both sub-cases for Case 1, it holds that 
        \begin{align*}
            &\Prob{X_t - X_{t+1} >0 \mid X_0, \dots, X_t} 
            =\frac{1}{2} 
            >\frac{1}{4} 
            \ge\Prob{X_t - X_{t+1} <0 \mid X_0, \dots, X_t}
        \end{align*}
        and especially
        \begin{align*}
            &\Prob{X_t - X_{t+1} >0 \mid X_0, \dots, X_t} 
            -\Prob{X_t - X_{t+1} <0 \mid X_0, \dots, X_t} 
            \ge \frac{1}{4}.
        \end{align*}

\textbf{Case 2:} $\col_t$ is not the globally least frequent color, so $r_{\col_t}=2$.
    
    This means that the least frequent color in total lies in the other partition, but can no longer be reduced. This is the case when all vertices in that partition are already colored in the same color. $\col_t$ is therefore the least frequent color that can still be reduced.
    
    Now, we get
    \begin{align*}
        \Prob{X_t - X_{t+1} >0 \mid X_0, \dots, X_t} = \frac{1}{2^{r_{\col_t}}}\ge \frac{1}{4}.
    \end{align*}
    
    Note that the reason for $\frac{1}{4}$ being a lower bound instead of the exact value is that $\col_t$ might become the least frequent color in total, increasing its rank to one and the probability of being chosen for reduction to $\frac{1}{2}$. As this case only improves our chances of moving in the correct direction and has no other effects, it can be disregarded.
    
    On the other hand, the probability of $\col_t$ increasing is, as above
    \begin{align*}
    &\Prob{X_t - X_{t+1} <0 \mid X_0, \dots, X_t} = \Prob{A} \cdot \Prob{B \mid A}.
    \end{align*}
    
    Again, we differentiate between two further cases:

\textbf{Case 2a:} There is only one other color in the same partition as $\col_t$, so $k' = 2$.

        As there is only one other color (whose rank must be three, as there are only three colors in total in the graph) left in the same partition as $\col_t$, it follows that
        \begin{align*}
            \Prob{A} = \frac{1}{8}.
        \end{align*}
        Further, as in Case 1a,
        \begin{align*}
            \Prob{B \mid A} =1.
        \end{align*}
        Therefore
        \begin{align*}
            \Prob{X_t - X_{t+1} <0 \mid X_0, \dots, X_t} = \frac{1}{8}.
        \end{align*}

\textbf{Case 2b:} There are more than two colors in the same partition as $\col_t$, so $k' > 2$.
    
        We introduce an event $C$, denoting the color of rank one being chosen for reduction, and analogously, an event $D$ for the color of rank $2$, which is $col_t$, being chosen.
        
        Now, the probability of choosing another color in the same partition as $col_t$ (event $A$) is less-equal the probability of not choosing either the color of rank one (event $C$) or of rank two (event $D$), meaning
        \begin{align*}
        \Prob{A} 
        \le 1 - \Prob{C} - \Prob{D} 
        = 1- \frac{1}{2}- \frac{1}{4} 
        =\frac{1}{4}.
        \end{align*}
        With $\Prob{B \mid A}$ calculated as in Case 1b and $k' > 2$, it holds that
        \begin{align*}
            \Prob{B \mid A} =\frac{1}{k'-1} \le \frac{1}{2}.
        \end{align*}
        It follows that
        \begin{align*}
           \Prob{X_t - X_{t+1} <0 \mid X_0, \dots, X_t} \le \frac{1}{4} \cdot \frac{1}{2} = \frac{1}{8}. 
        \end{align*}
    
    This concludes Case 2b.
    Therefore, in both sub-cases for Case 2, it holds that 
    \begin{align*}
       &\Prob{X_t - X_{t+1} >0 \mid X_0, \dots, X_t} 
       \ge\frac{1}{4}
       >\frac{1}{8} \\
       \ge~&\Prob{X_t - X_{t+1} <0 \mid X_0, \dots, X_t} 
    \end{align*}
    and especially 
    \begin{align*}
        &\Prob{X_t - X_{t+1} >0 \mid X_0, \dots, X_t} - \Prob{X_t - X_{t+1} <0 \mid X_0, \dots, X_t} \\
        \ge~&\frac{1}{8},
    \end{align*}
which concludes Case 2. 
Since $$\Prob{X_t - X_{t+1} >0 \mid X_0, \dots, X_t} - \Prob{X_t - X_{t+1} <0 \mid X_0, \dots, X_t}\ge\frac{1}{8}$$ holds in both cases, the lemma is proven.
\end{proof}

Now, using this, we prove \Cref{lem:ct_to_0}.

\begin{proof} [Proof of \Cref{lem:ct_to_0}]
Let $T_k$ denote the time when the process $X_t$ reaches 0 on level $k$. 

It holds that for all $t \le T_k$, $X_t \ge 0$, as no negative amount of vertices can be colored in a given color.

Further, it follows from \Cref{lem:ct_reduced} that 
\begin{align*}
    \Ex{X_t - X_{t+1} \mid X_0, \dots, X_t} \ge \frac{1}{8}.
\end{align*}

With this, it follows from \Cref{thm:additive_drift} [\nameref{thm:additive_drift}] that
\begin{align*}
    \Ex{T_k} \le 8 \cdot \Ex{X_0}.
\end{align*}

We note that $X_0 \le \frac{n}{{k}}$, and therefore $\Ex{X_0} \le \frac{n}{{k}}$. It follows that 
\begin{align*}
    \Ex{T_k} \le 8 \cdot \frac{n}{{k}},
\end{align*}
which proves the lemma.
\end{proof}
Finally, we can use this to calculate the total run time. 
Across the entire run of the algorithm ${k}$ decreases from $n$ to two and there are $n-2$ iterations. 
Therefore
\begin{align*}
&\sum_{{k}=2}^{n} \Ex{T_k} 
\le\sum_{{k}=2}^{n} 8 \cdot \frac{n}{{k}} 
=8n \cdot \sum_{{k}=2}^{n} \frac{1}{{k}} 
<8n \cdot  \sum_{{k}=1}^{n} \frac{1}{{k}}.
\end{align*}

With $\sum_{{k}=1}^{n} \frac{1}{{k}} \in \Theta (\log n)$, the total run time lies in $\mathcal{O}(n \log n )$.
\end{proof}

\section{Lower Bound}

After proving an upper bound for gray-box local search on complete bipartite graphs, we now prove a lower bound.

\begin{theorem} \label{the:lower_bound_gbkmn}
    Gray-box local search takes $\Omega (n \log n)$ fitness evaluations in expectation to find a minimal coloring on a complete bipartite graph with $n$ vertices.
\end{theorem} 

For this, let us first assume that we know in advance in which order colors will be removed from the coloring, and that we can number them accordingly. Color $0$ is therefore the last color to be removed from the coloring and color $n-3$ the first color to be removed. The two colors remaining in the final coloring do not receive numbers.

Now, our goal is to calculate an expected number of re-colorings a given vertex will go through before reaching its final color, which is either of the two colors remaining in the graph in the end. 

Here, we first make several assumptions and argue why none of these assumptions violate generality.
\begin{itemize}
    \item For ease of notation, we assume that any vertex can be flipped to any other color in use in the current coloring at any time without creating an improper coloring. While there is always at least one color (namely any color currently in use in the opposite partition) that we cannot flip to, at least one of the partitions must have $\Theta (k)$ colors to choose from for flipping, with $k$ denoting the number of colors in use in the current coloring. Since we are only interested in the asymptotic run time, it has no bearing on the resulting run time to disregard which partition a color appears in.
    \item A vertex colored in color $i$ can only be re-colored to a color $j$ with $j > i$. This assumption yields a lower bound on the run time, as we remove a number of possible unnecessary re-colorings.
\end{itemize}

We now define a recursive term $w$ for calculating the expected number of additional re-colorings we need for a vertex of color $i$ to be re-colored to one of the final colors.

We define
\begin{align*}
    w(0) = 0;~
    w(i) = \frac{i}{i+2} + \frac{1}{i+2} \sum_{j=0}^{i-1} w(j).
\end{align*}

This is because the last color to be removed (color $0$) will require no additional re-colorings, as it can only be flipped to either of the two final colors. 

For the recursive term, we have $i+2$ colors to choose from for flipping, meaning we have a probability of $\frac{2}{i+2}$ of flipping to either of the final colors. In turn, this leads to a probability of $\frac{i}{i+2}$ of flipping to one of the non-final colors, leading to an expected additional number of re-colorings of $\frac{i}{i+2}$. 

Further, we need to account for any additional re-colorings that arise from the new color we flip to. We have a probability of $\frac{1}{i+2}$ of choosing each of the non-final colors. For each, we then also need to add the number of expected re-colorings, meaning $\sum_{j=0}^{i-1} w(j)$.

An example of possible transitions to illustrate how the recursive term comes about is given in \Cref{fig:lower_bound_example}.

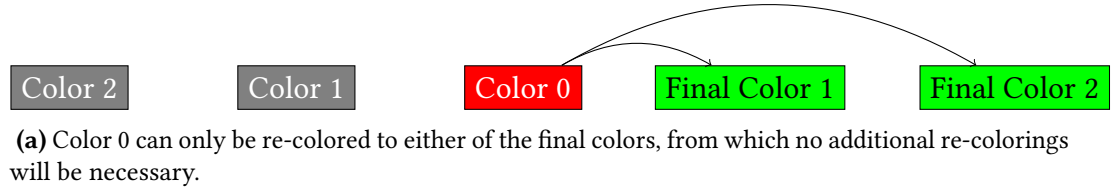
\begin{figure}[hbt!]
    \centering
    \caption{An example of the possible transitions in a graph with $5$ vertices.}
    \label{fig:lower_bound_example}
         \begin{subfigure}[b]{1\textwidth}
            \begin{tikzpicture}
        
                    \node[shape=rectangle,draw=black,fill=gray,text=white] (C2) at (0, 0) {Color 2};
                    \node[shape=rectangle,draw=black,fill=gray,text=white] (C1) at (3, 0) {Color 1};
                    
                    \node[shape=rectangle,draw=black,fill=green,text=black] (F1) at (9, 0) {Final Color 1};
                    \node[shape=rectangle,draw=black,fill=green,text=black] (F2) at (12.5, 0) {Final Color 2};

                    \node[shape=rectangle,draw=black,fill=red,text=white] (C0) at (6, 0) {Color 0}
                        edge[bend left][->] node {} (F1)
                        edge[bend left][->] node {} (F2);

            \end{tikzpicture}
            \caption{Color 0 can only be re-colored to either of the final colors, from which no additional re-colorings will be necessary.}
             
        \end{subfigure}
        \vfil
         \begin{subfigure}[b]{1\textwidth}
            \begin{tikzpicture}
        
                    \node[shape=rectangle,draw=black,fill=gray,text=white] (C2) at (0, 0) {Color 2};

                    \node[shape=rectangle,draw=black,fill=green,text=black] (F1) at (9, 0) {Final Color 1};
                    \node[shape=rectangle,draw=black,fill=green,text=black] (F2) at (12.5, 0) {Final Color 2};

                    \node[shape=rectangle,draw=black,fill=yellow,text=black] (C0) at (6, 0) {Color 0};

                    \node[shape=rectangle,draw=black,fill=red,text=white] (C1) at (3, 0) {Color 1} 
                        edge[bend left][->] node {} (C0)
                         edge[bend left][->] node {} (F1)
                        edge[bend left][->] node {} (F2);

            \end{tikzpicture}
            \caption{Color 2 has a probability of $\frac{1}{3}$ of being re-colored to color 0 (meaning 1 additional re-coloring) and a probability of $\frac{2}{3}$ of being re-colored to either of the final colors.}
             
        \end{subfigure} 
    \vfil
         \begin{subfigure}[b]{1\textwidth}
            \begin{tikzpicture}
                    \node[shape=rectangle,draw=black,fill=yellow,text=black] (C1) at (3, 0) {Color 1};
                    \node[shape=rectangle,draw=black,fill=yellow,text=black] (C0) at (6, 0) {Color 0};
                    \node[shape=rectangle,draw=black,fill=green,text=black] (F1) at (9, 0) {Final Color 1};
                    \node[shape=rectangle,draw=black,fill=green,text=black] (F2) at (12.5, 0) {Final Color 2};

                    \node[shape=rectangle,draw=black,fill=red,text=white] (C2) at (0, 0) {Color 2}
                        edge[bend left][->] node {} (C1)
                        edge[bend left][->] node {} (C0)
                        edge[bend left][->] node {} (F1)
                        edge[bend left][->] node {} (F2);
            \end{tikzpicture}
            \caption{Color 3 has a probability of $\frac{1}{4}$ of being re-colored to color 0 (meaning 1 additional re-coloring) and color 1 (meaning in expectation $1+\frac{1}{3}$ additional re-colorings) and a probability of $\frac{2}{4}$ of being re-colored to either of the final colors.}
             
        \end{subfigure}
\end{figure}   

Using this recursive function and given that we start with one vertex of each color, the total number of re-colorings is given by $$n-2+ \sum_{i=0}^{n-3} w(i).$$

Here, we have $n-2$ re-colorings that are always necessary to re-color all vertices that start out in a non-final color to either of the two final colors. With $\sum_{i=0}^{n-3} w(i)$, we then add the number of additional re-colorings performed for each vertex. Since we start with one vertex of each color, we must add the number of re-colorings for each possible starting color once.

Before further considering that sum, we first bound $w(i)$.

\begin{lemma}
   For all $i \in [n-2]$, it holds that $w(i)=\sum_{j=0}^{i-1} \frac{1}{j+3}$.
\end{lemma}
\begin{proof}
We prove this by induction over $i$. 
For $i=1$, we have
\begin{align*}
     w(1) = \frac{1}{1+2} + \frac{1}{1+2} \sum_{j=0}^{1-1} w(j) 
     = \frac{1}{3} + \frac{1}{3} w(0) 
     = \frac{1}{3} 
     = \sum_{j=0}^{1-1} \frac{1}{j+3},
\end{align*}
which proves the statement for $i=1$.

Now, assume that $w(i)=\sum_{j=0}^{i-1} \frac{1}{j+3}$ holds. We note that 
\begin{align*}
    w(i) &= \frac{i}{i+2} + \frac{1}{i+2} \sum_{j=0}^{i-1} w(j) 
\end{align*}
can be likewise written as
\begin{align*}
    \sum_{j=0}^{i-1} w(j)  &= (i+2) \cdot \left(w(i)- \frac{i}{i+2}\right). &&(*)
\end{align*}
By definition
\begin{align*}
    w(i+1) &= \frac{i+1}{i+3} + \frac{1}{i+3} \sum_{j=0}^{i} w(j) \\
    &= \frac{i+1}{i+3} + \frac{1}{i+3} \cdot \left(w(i) + \sum_{j=0}^{i-1} w(j)\right).
\end{align*}
Now, using $(*)$, we get
\begin{align*}
    w(i+1) &= \frac{i+1}{i+3} + \frac{1}{i+3} \cdot \left( w(i) + \left((i+2) \cdot \left(w(i)- \frac{i}{i+2} \right) \right) \right) \\
    &= \frac{i+1}{i+3} + \frac{1}{i+3} \cdot \left( w(i) \cdot (i+3) - \frac{i (i+2)}{i+2} \right) \\
    &= \frac{i+1}{i+3} + w(i)  - \frac{i}{i+3} \\
    &= \frac{1}{i+3} + w(i).
\end{align*}
With $w(i)=\sum_{j=0}^{i-1} \frac{1}{j+3}$, this is equal to 
\begin{align*}
    w(i) = \frac{1}{i+3} + \sum_{j=0}^{i-1} \frac{1}{j+3} 
    = \sum_{j=0}^{i} \frac{1}{j+3},
\end{align*}
which proves our statement and thereby the lemma.
\end{proof}

Now, we use this to prove the theorem. 
\begin{proof} [Proof of \Cref{the:lower_bound_gbkmn}]

We note again that the total number of re-colorings is given by
\begin{align*}
    n-2+ \sum_{i=0}^{n-3} w(i),
\end{align*}
as explained above.
It holds that $$w(i)=\sum_{j=0}^{i-1} \frac{1}{j+3}=\sum_{j=3}^{i+2} \frac{1}{j}= \sum_{j=1}^{i+2} \frac{1}{j} - 1 - \frac{1}{2}.$$ Using Equation 1.4.2 by \citeauthor{doerr_probabilistic_2020} \cite{doerr_probabilistic_2020}, we know that $$\sum_{j=1}^{i+2} \frac{1}{j} > \log(i+2).$$ Therefore, $w(i) > \log(i+2)- 1 - \frac{1}{2}.$
Using this, we get
\begin{align*}
    n-2+  \sum_{i=0}^{n-3} w(i) &> n-2+  \sum_{i=0}^{n-3}\bigg(\log(i+2)- 1 - \frac{1}{2}\bigg) \\
    &= 2n-5+\frac{n-3}{2} \sum_{i=1}^{n-3}\log(i+2).
\end{align*}

With $\sum_{i=1}^{n-3}\log(i+2) \in \Theta (n \log n)$, this lies in $\Omega(n \log n)$, proving the theorem.
\end{proof}

    \chapter{Conclusions \& Outlook}
    In this thesis, we analyzed the performance of local search operators on the vertex coloring problem for bipartite graphs. We presented results on the local search fitness landscapes of different bipartite graphs as well as the performance of a gray-box mutation operator on complete bipartite graphs.

During the landscape analysis, we explored several structures in graphs that lead to local search succeeding or failing in finding a global optimum. We found structural features that ensure no non-global local optima exist, which means that local search will find an optimal coloring.

However, we also identified two substructures that lead to the existence of non-global local optima, namely crown graphs and $3$-circles. It follows that local search may fail to find an optimal coloring on a large number of bipartite graphs, which proves that local search in itself cannot be reliably used as a strategy for finding optimal vertex colorings. Further, specifically the non-global local optima caused by crown graphs can be arbitrarily far from the global optimum, meaning that local search does not even reliably offer a good approximation of the optimal coloring.

Regarding the run time of local search, we developed a gray-box mutation operator for local search that works by reducing less-frequent colors with higher probability. We analyzed its performance on complete bipartite graphs and showed that its run time lies in $\Theta(n \log n)$, which is a drastic improvement on the exponential run time of the black-box Random Local Search. 

With our work, we lay important groundwork regarding the capabilities of local search on vertex coloring. We identified limitations caused by graph structures, but also showed that using gray-box mutation causes a drastic improvement in run time on instances where a global optimum can be found.

While our work has shown that local search is not a generally successful strategy for coloring graphs, we believe it may be possible to combine gray-box mutation operators such as the one we designed with strategies for escaping local optima for better results on bipartite graphs. This might either happen in the form of a global search gray-box mutation operator, or by using algorithmic strategies for escaping local optima such as iterated local search or Variable Neighborhood Search. 





    \makeatletter
        \def\toclevel@chapter{-1}
        \def\toclevel@section{0}
    \makeatother

    \pagestyle{plain}

    \renewcommand*{\bibfont}{\small}
    \printbibheading
    \addcontentsline{toc}{chapter}{Bibliography}
    \printbibliography[heading = none]

    \addchap{Declaration of Authorship}
    I hereby declare that this thesis is my own unaided work. All direct or indirect sources used are acknowledged as references.\\[6 ex]

\begin{flushleft}
    Potsdam, \today
    \hspace*{2 em}
    \raisebox{-0.9\baselineskip}
    {
        \begin{tabular}{p{5 cm}}
            \hline
            \centering\footnotesize\printAuthor
        \end{tabular}
    }
\end{flushleft}

\end{document}